\pdfoutput=1

\documentclass[letterpaper]{article} 
\usepackage{aaai18}  
\usepackage{times}  
\usepackage{helvet}  
\usepackage{courier}  
\usepackage{url}  
\usepackage{graphicx}  
\usepackage[ruled,linesnumbered]{algorithm2e}
\usepackage{booktabs}
\usepackage{amsthm}
\usepackage{amsfonts}
\usepackage{amssymb}
\usepackage{mathtools}
\usepackage{thmtools}
\usepackage{thm-restate}
\usepackage{paralist}
\usepackage{ifdraft}
 
\declaretheorem[name=Theorem]{theorem}
\declaretheorem[sibling=theorem,name=Lemma]{lemma}
\declaretheorem[sibling=theorem,name=Proposition]{proposition}
\declaretheorem[sibling=theorem,name=Claim]{claim}
\declaretheorem[sibling=theorem,name=Corollary]{corollary}
\declaretheorem[sibling=theorem,name=Definition]{definition}
\declaretheorem[sibling=theorem,name=Example]{example}

\newif\ifextendedversion
\extendedversiontrue 

\newcounter{ObservationCounter}

\renewcommand\vec{\mathbf}

\newcommand{\observation}[1]{%
    \smallskip\par\noindent%
    \refstepcounter{ObservationCounter}\label{#1}%
    \emph{Observation~\arabic{ObservationCounter}. }%
}

\newcommand{\proofidea}{Proof sketch}
\newcommand{\myparagraph}[1]{\smallskip \noindent \textbf{#1}\ }

\newcommand{\head}{\ensuremath{\operatorname{head}}}
\newcommand{\body}{\ensuremath{\operatorname{body}}}

\newcommand{\program}{\ensuremath{\Pi}}

\newcommand{\dclass}{\ensuremath{\mathcal{D}}}
\newcommand{\qclass}{\ensuremath{\mathcal{Q}}}

\newcommand{\leftsquigarrow}{\ensuremath{%
    \mathrel{\raisebox{1pt}{\ensuremath{%
        \mathbin{\rotatebox[origin=c]{180}{\ensuremath{%
            \leadsto
        }}}%
    }}}%
}}

\newcommand{\restrict}[2]{\ensuremath{#1{\restriction_{#2}}}}

\newcommand{\lang}{\ensuremath{\mathcal{L}}}

\newcommand{\ucontained}{\ensuremath{\sqsubseteq^{\mathrm{u}}}}

\newcommand{\qclassfont}[1]{\textsc{#1}}
\newcommand{\fpclass}{\qclassfont{fp}}
\newcommand{\ogclass}{\qclassfont{og}}
\newcommand{\nrclass}{\qclassfont{nr}}
\newcommand{\ognrclass}{\qclassfont{ognr}}
\newcommand{\dnrclass}{\textsc{dnr}}
\newcommand{\initstate}{\ensuremath{\mathit{init}}}
\newcommand{\datalogmtl}{DatalogMTL}
\newcommand{\inputlineref}{3}
\newcommand{\materlineref}{4}
\newcommand{\outputlineref}{5}

\newcommand{\ccomplement}{\textnormal{co}}
\newcommand{\chard}{\textnormal{-hard}}
\newcommand{\ccomplete}{\textnormal{-complete}}

\newcommand{\pspace}{\textsc{PSpace}}
\newcommand{\pspacehard}{\pspace\chard}
\newcommand{\pspacecomplete}{\pspace\ccomplete}
\newcommand{\expspace}{\textsc{ExpSpace}}
\newcommand{\expspacehard}{\expspace\chard}
\newcommand{\expspacecomplete}{\expspace\ccomplete}
\newcommand{\exptime}{\textsc{Exp}}

\newcommand{\exptimecomplete}{\exptime\ccomplete}
\newcommand{\ptime}{\textsc{P}}

\newcommand{\ptimecomplete}{\ptime\ccomplete}
\newcommand{\nptime}{\textsc{NP}}

\newcommand{\logspace}{\textsc{LogSpace}}

\newcommand{\aczero}{\textsc{AC}\ensuremath{^0}}

\newcommand{\nexptime}{\textsc{NExp}}

\newcommand{\nexptimecomplete}{\nexptime\ccomplete}
\newcommand{\conexptime}{\ccomplement\textsc{NExp}}
\newcommand{\conexptimehard}{\conexptime\chard}
\newcommand{\conexptimecomplete}{\conexptime\ccomplete}
\newcommand{\conptime}{\ccomplement\nptime}
\newcommand{\conptimehard}{\ccomplement\nptime\chard}
\newcommand{\conptimecomplete}{\ccomplement\nptime\ccomplete}

\newcommand{\containment}{\textsc{Cont}}
\newcommand{\bodcontainment}{\textsc{Cont}_{\mathcal{O}}}
\newcommand{\window}{\textsc{Window}}
\newcommand{\bodwindow}{\textsc{Window}_{\mathcal{O}}}
\newcommand{\evaluation}{\textsc{Eval}}

\newcommand{\tmin}{\ensuremath{\tau_\mathrm{min}}}

\ifdraft{
    \usepackage{color} 
    \usepackage[dvipsnames]{xcolor}
}{}

\nocopyright{}

\begin{document}

\title{The Window Validity Problem in Rule-Based Stream Reasoning}
\author{Alessandro Ronca, Mark Kaminski, Bernardo Cuenca Grau \and Ian Horrocks\\
      Department of Computer Science, University of Oxford, UK\\
      $\{$alessandro.ronca, mark.kaminski, bernardo.cuenca.grau, ian.horrocks$\}$@cs.ox.ac.uk}

\maketitle

\begin{abstract}
    Rule-based  temporal query languages provide the expressive power and flexibility required to capture in a natural way
complex analysis tasks over streaming data. 
Stream processing applications, however, typically require near real-time
response using limited resources. In particular, it becomes essential that
the underpinning query language has favourable computational properties and that stream processing algorithms  are able to
keep only a small number of previously received facts in memory at any point in time without sacrificing correctness.
In this paper, we propose a recursive fragment of temporal Datalog with tractable data complexity and study the properties of a generic stream reasoning
algorithm for this fragment.
We focus on the \emph{window validity problem} as a way to minimise the number of time points for 
which the stream reasoning algorithm needs to keep data in memory at any point
in time.

\end{abstract}

\section{Introduction}

Query processing over streams is becoming increasingly 
important for data analysis in domains as diverse as 
financial trading  \cite{nuti2011algorithmic}, equipment maintenance
\cite{cosad2009wellsite}, or network security  \cite{munz2007real}.

A growing body of research has recently focused on
extending traditional stream management systems with
reasoning capabilities \cite{barbieri2010incremental,calbimonte2010enabling,anicic2011ep,le2011native,zaniolo2012streamlog,ozcep2014stream,beck2015lars,dao2015comparing,ronca2017stream}.  
Languages well-suited for stream reasoning applications are typically rule-based, where prominent examples include  \emph{temporal Datalog}  \cite{chomicki1988temporal} 
and \emph{\datalogmtl{}}  \cite{DBLP:conf/aaai/BrandtKKRXZ17}.
These core languages are powerful enough to
capture many other temporal formalisms \cite{abadi1989temporal,baudinet1992temporal}
and provide the logical underpinning for other expressive languages
proposed in the stream reasoning literature \cite{zaniolo2012streamlog,beck2015lars}.

Rules provide the expressive power and flexibility required to naturally capture in a declarative way
complex analysis tasks over streaming data.
This is illustrated by the following example in network security, where  intrusion detection policies (IDPs) are represented
in temporal Datalog.

\begin{example}\label{ex:running}
Consider a computer network which is being monitored for external threats.
Bursts (unusually high amounts of data) between any pair of nodes in the network
are detected by specialised monitoring devices and streamed to the network's management centre as timestamped facts. 
A monitoring task in the  centre is to identify nodes that may have been hacked according to a specific  IDP, and add them to a 
blacklist of nodes.
In this setting, one may want to know the contents of the blacklist 
at any given point in time in order to decide on further action. This task is captured by a temporal Datalog query consisting of the 
rules given next and where $\mathsf{Black}$ is the designated output predicate:
\begin{align}
 \mathsf{Brst}(x,y,t) \land  \mathsf{Brst}(z,y,t+1) &\to
        \mathsf{Attk}(x,y,t+1)   \label{eq:burst} \\
   \mathsf{Attk}(x,y,t) \land \mathsf{Attk}(x,y,t+1) &  \nonumber \\
      {}\land  \mathsf{Attk}(x,y,t+2) &\to \mathsf{Black}(x,t+2) \label{eq:attck} \\
    \mathsf{Black}(x,t) &\to \mathsf{Black}(x,t+1)  \label{eq:black} \\
    \mathsf{Attk}(x,y,t) &\to \mathsf{Grey}(x,\mathit{max}, t)  \label{eq:grey1} \\
    \mathsf{Grey}(x,i,t) \land \mathsf{Succ}(j,i) &\to
        \mathsf{Grey}(x,j,t+1)  \label{eq:grey2} \\
        \mathsf{Grey}(x,i,t) \land \mathsf{Brst}(x,y,t) &\to \mathsf{Black}(x,t)  \label{eq:blackfinal}
\end{align}
Rule~\eqref{eq:burst} identifies two consecutive bursts from nodes $v$ and $v''$
to a node $v'$ in the network as an attack on $v'$ originated by
$v$. Rule~\eqref{eq:attck} implements an IDP where three consecutive attacks
from $v$ on $v'$ result in $v$ being added to the blacklist, where it remains
indefinitely (Rule~\eqref{eq:black}).
Rules~\eqref{eq:grey1}--\eqref{eq:blackfinal} implement a second IDP where an
attack from $v$ on any node leads to $v$ being identified as suspicious and
added to a ``greylist''. Such list comes with a succession of decreasing
warning levels, where the maximum is represented by the constant $max$ and
where the relationship from each level to the next is captured by a binary,
non-temporal, $\mathsf{Succ}$ predicate. As time goes by, the warning level
decreases; however, if at any point during this process node $v$ generates
another burst to any other node in the network, then it gets blacklisted.
\end{example}

Stream processing applications typically require near real-time response using limited resources; this  becomes
especially challenging in the context of rule-based stream reasoning due to the following reasons:
\begin{enumerate}[1.]
  \item  Fact entailment over temporal rule languages is typically intractable
      in data complexity---$\expspacecomplete$ in the case of \datalogmtl{} and  $\pspacecomplete$ in the case of
  temporal Datalog. Furthermore, known tractable fragments are non-recursive  \cite{ronca2017stream,DBLP:conf/aaai/BrandtKKRXZ17}, which limits their applicability to certain data analysis applications. 
  \item In order to adhere to memory limitations and scalability requirements, systems can only keep a limited history of previously received
input facts in memory to perform further computations. Rules, however, can propagate derived information both towards past and future time points and hence  query answers 
can depend  on data that has not yet been received 
as well as on data
that arrived far in the past. This may force the system to 
keep in memory a very  large (or even unbounded) input history  to ensure correctness.
\end{enumerate}
We address the first challenge by introducing in Section \ref{sec:fpqueries}
the language of \emph{forward-propagating queries}---a fragment of temporal
Datalog that extends plain (non-temporal) Datalog by allowing unrestricted
recursive propagation of information into future time points, while at the same
time precluding propagation of derived facts towards past time points.  Our
language is sufficiently expressive to capture interesting analysis tasks over
streaming data, such as the one illustrated in our previous example. Moreover,
we show that forward-propagating queries can be answered in polynomial time in
the size of the input data and hence they are well-suited for data-intensive
applications.

To address the second challenge,
we take as a starting point a generic algorithm which accepts as input a set of non-temporal background facts and a stream of timestamped facts
and outputs as a stream the answers to a forward-propagating query $Q$. The algorithm is parametrised by 
a window size  $w$ and a signature $\Sigma$, which 
determine the set of facts stored in memory by the algorithm at any point in time. 
As the algorithm receives the input stream at  time point $\tau$, it 
computes all implicit $\Sigma$-facts and answers to $Q$ holding at $\tau$ using only the facts held in memory, and subsequently 
discards all stored facts holding at $\tau-w$. 
For the algorithm to be correct,  the computed answers for each $\tau$ over the restricted set of facts in memory 
must coincide with the answers over the entire stream. Such an assurance, however, 
can only be given for certain values of $\Sigma$ and if the window parameter $w$
is large enough so that facts that may influence answers at later time points are not discarded too early.
This motivates the \emph{window validity problem},  which is to decide whether a given window $w$ is valid for a given query $Q$ and signature $\Sigma$  in the sense that
the aforementioned correctness guarantee
holds for any input data.

In  our prior work \cite{ronca2017stream}, we  considered an instantiation of
the generic stream reasoning algorithm where only explicit facts from the input stream (and hence no entailed facts) are kept in memory. 
This setting can, however, be problematic in the presence of recursion, in that
a recursive query may not admit a valid window. Stream reasoning clearly becomes  impractical for such queries
since the entire stream received so far must be kept in memory by the algorithm in order to ensure
correctness. 

To address this limitation, we consider in Section \ref{sec:mbstreamreasoning}
a \emph{full materialisation} variant of the algorithm in which all facts
(explicit or implicit) over the entire signature are kept in memory; as a
result, when a fact is discarded by the algorithm, its consequences at later
time points are not lost.  In this setting, we can show that a valid window is
guaranteed to exist for any forward-propagating query, and a (possibly larger
than needed) window can be obtained syntactically by inspection of the query.
From a practical perspective, however, it is important to have a valid window
that is as small as possible since the number of facts entailed by the query's
rules at any given time point can be very large. Thus, we investigate in
Sections~\ref{sec:window} and~\ref{sec:boundedobject} the computational
properties of window validity in this revised setting.

In Section~\ref{sec:window}, we show that window validity and query containment are interreducible problems, and hence known complexity bounds on temporal query containment transfer directly. 
 In particular, undecidability of window validity for forward-propagating queries follows  from the undecidability of query containment for non-temporal Datalog.
  
 To regain decidability, we consider in Section~\ref{sec:boundedobject} the situation where the set of relevant domain objects can be fixed in advance, in the sense that
 input facts can refer only to those objects. 
 In Example~\ref{ex:running}, this assumption amounts to fixing both the 
 nodes in the network and the grey list's warning levels, 
  and requiring that all input facts mention only these
 objects.
 This assumption allows us to ground  the non-temporal variables of the query to
 a set of known objects; such grounding is exponential and results
 in an \emph{object-ground} query where all  variables are temporal.
 We show that
 the window problem is $\pspacecomplete$ for object-ground forward-propagating
 queries and $\conptimecomplete$ if the query is also
 non-recursive. This immediately gives us an $\expspace$ upper bound ($\conexptime$ if queries are additionally assumed to be non-recursive) for the fixed-domain window validity problem; we then prove that these bounds are tight.
 Our results show that, although window validity is undecidable, we can obtain decidability
under reasonable assumptions on the input data. Even under such assumptions, the problem is computationally intractable; however, 
queries can be assumed to be relatively small in practice and windows can be
computed offline, prior to receiving any data. 

Finally, for applications where one cannot assume the object domain to be fixed
in advance, we propose in Section~\ref{sec:sufficient} a sufficient condition
for the validity of a window that can be checked in exponential time without
additional assumptions.

\ifextendedversion
    Complete proofs of all our technical results are deferred to the appendix.
\fi


\section{Preliminaries}
\label{sec:preliminaries}

We recapitulate temporal Datalog \cite{chomicki1988temporal}  as a basic
language for stream reasoning. 

\smallskip

\noindent \textbf{Syntax}\ 
A signature consists of predicates, constants and 
variables, where constants are partitioned into objects and 
non-negative integer time points and variables are partitioned into 
object variables and time variables.
An object term is an object or an object variable.
A time term is a time point, a time variable, or an expression
of the form $t + k$ with $t$ a time variable, $k$  an integer, 
and $+$ the integer addition function.

Predicates are partitioned into 
\emph{extensional} (EDB) and 
\emph{intensional} (IDB) and they come with a non-negative integer
\emph{arity} $n$, where each position ${1 \leq i \leq n}$ is of either
\emph{object} or \emph{time sort}. 
A predicate is \emph{rigid} if
all its positions are of object sort and it is \emph{temporal} 
if the last position is 
of time sort and all other positions are of object sort.
An  \emph{atom} is an expression $P(s_1, \ldots,
s_n)$ where $P$ is a $n$-ary predicate and each $s_i$ is a term of the required
sort; we sometimes use the term $P$-atom to refer to an atom with predicate $P$.
A \emph{rigid} atom (respectively, temporal, IDB,
EDB) is an atom over a rigid predicate
(respectively, temporal, IDB, EDB).

A  \emph{rule} $r$ is of the form $\bigwedge_i \alpha_i
\rightarrow \alpha$, where $\alpha$ and each $\alpha_i$ are rigid or temporal
atoms, and $\alpha$ is IDB whenever $\bigwedge_i \alpha_i$ is non-empty.
Atom $\head(r) = \alpha$ is the \emph{head} of $r$, and 
$\body(r) = \bigwedge_i \alpha_i$ is the \emph{body} of $r$.
Rules are \emph{safe}---that is, all variables occur 
in the body. 
A  \emph{program} $\program$ is a finite set of rules.
A term, atom, rule, or program is \emph{ground} if it has no variables.
A predicate $P$ is
$\program$\emph{-dependent} on predicate $P'$ if $\program$ has a rule
with $P$ in the head and $P'$ in the body.
A \emph{fact}  is a ground, function-free 
rigid or temporal atom, and a \emph{dataset} is a (possibly infinite) set of EDB
facts. 
Each fact $\alpha$ corresponds to a rule having empty body and $\alpha$ in the head, 
so we use $\alpha$ and its corresponding rule interchangeably.

A  \emph{query} is a pair $Q = \langle P_Q, \Pi_Q \rangle$ with 
$\Pi_Q$ a program and
$P_Q$ an IDB \emph{output predicate} in $\Pi_Q$ not occurring in the body of any rule in $\Pi_Q$. 
We also denote with $\Sigma_Q$ the set of all IDB predicates in $\Pi_Q$.
Query $Q$ is
\begin{compactitem}
\item[--] \emph{temporal}  if $P_Q$ is a temporal predicate;
 \item[--] \emph{Datalog} if no temporal predicate occurs in $\Pi_Q$;
 \item[--] \emph{object-ground} if $\Pi_Q$ has no object variables; and
\item[--] \emph{non-recursive} if the directed graph induced by the
$\Pi_Q$-dependencies is acyclic. 
\end{compactitem}

\myparagraph{Semantics}
Rules are interpreted as universally quantified
first-order sentences. A Herbrand interpretation $\mathcal{H}$ is a
(possibly infinite) set of facts. It  \emph{satisfies} a 
rigid atom $\alpha$ if $\alpha \in \mathcal{H}$, and it satisfies
a temporal atom $\beta$ if evaluating the addition 
function in $\beta$ yields a fact in $\mathcal{H}$.
Satisfaction is extended to conjunctions of ground atoms,
rules and programs in the standard way.
If $\mathcal{H} \models \program$, then $\mathcal{H}$ is a \emph{model} of $\program$.
Program $\program$ \emph{entails} a fact $\alpha$, 
written $\program \models \alpha$, 
if $\mathcal{H} \models \program$ implies $\mathcal{H} \models \alpha$. 
The set of \emph{answers} to a query $Q$ over a dataset $D$, written $Q(D)$, 
consists of each $P_Q$-fact $\alpha$ such that $\Pi_Q \cup D \models \alpha$. 

\myparagraph{Reasoning}
We next define two basic reasoning problems, which we parametrise 
to specific classes of  input queries $\qclass$ and
datasets $\dclass$.
Similarly to \cite{chomicki1988temporal}, 
we assume from now onwards in all reasoning problems that numbers in input
queries and datasets  are coded in unary; our complexity results may (and almost
certainly will) change if binary encoding is assumed, and we leave this investigation for future work. 
Furthermore, we make the following general assumptions for  each $\dclass$:
\emph{(1)}  for each $D \in \dclass$ and each finite subset
$S$ of $D$ there is a finite $D'\in\dclass$ such that
$S\subseteq D'\subseteq D$; and
\emph{(2)} for each $D \in \dclass$ and unary temporal
fact $\alpha$, we have $D \cup \{\alpha\} \in \dclass$.
The former property is a form of compactness closure, whereas the latter is a closure property under addition of
unary temporal facts.

The \emph{query evaluation problem} $\evaluation_{\dclass}^{\qclass}$, for $\qclass$ a class of queries and
$\dclass$ a class of finite datasets, is to check whether
$\alpha \in Q(D)$ for $\alpha$ an input fact, $Q \in \qclass$ and $D \in \dclass$; the \emph{data complexity} of $\evaluation_{\dclass}^{\qclass}$
is the complexity for fixed $Q$.
Query evaluation for arbitrary datasets is
$\pspacecomplete$ in data complexity under unary encoding of numbers
\cite{chomicki1988temporal}, and in $\aczero$  for 
non-recursive queries.

Let $Q_1$ and $Q_2$ be queries having
the same output predicate.
Then,  $Q_1$ is \emph{contained} in  $Q_2$ with respect to $\dclass$, 
written $Q_1 \sqsubseteq_{\dclass} Q_2$, if $Q_1(D)
\subseteq Q_2(D)$ for each  $D \in \dclass$.
The \emph{containment problem} $\containment_{\dclass}^{\qclass}$  is to check
$Q_1 \sqsubseteq_{\dclass} Q_2$  for given $Q_1, Q_2 \in  \qclass$.
For simplicity, we drop $\dclass$  from $Q_1 \sqsubseteq_{\dclass} Q_2$ and 
$\containment_{\dclass}^{\qclass}$ 
(respectively, from $\evaluation_{\dclass}^{\qclass}$) whenever $\dclass$ is the
class of all datasets (respectively, of all finite datasets).

Our definition of containment considers
infinite datasets, which  is required to capture 
streams. This does not change the nature of the
problem due to the properties of first-order logic and our assumptions on $\dclass$; 
\ifextendedversion
    as shown in the appendix, 
\fi
$Q_1 \sqsubseteq_{\dclass} Q_2$ if and only if
$Q_1 \sqsubseteq_{\dclass'} Q_2$ with $\dclass'$ the class consisting of all finite datasets
in $\dclass$. By standard results in nontemporal Datalog, it follows that unrestricted containment
is undecidable \cite{shmueli1993equivalence}, and
it is $\conexptime$-hard for non-recursive queries \cite{benedikt2010impact}.

\section{Forward-Propagating Queries}
\label{sec:fpqueries}

Stream processing applications are  data-intensive, requiring fast  response using limited resources. 
Tractability of query evaluation in data complexity is thus  a key requirement for logics 
underpinning stream reasoning systems. 
Query evaluation in temporal Datalog is, however, $\pspacecomplete$ in data complexity, which limits its
applicability. 

In this section we introduce the language of forward-propagating queries---a  fragment of temporal Datalog
which allows unrestricted recursive propagation of derived facts into the present and future time points, while at the same time
precluding
propagation
towards past time points.

\begin{definition}\label{def:fp-query}
The \emph{offset}  of a time term $s$ equals zero if $s$ is a 
time variable, and it equals $k$ if $s$ is the time point $k$ or 
a time term of the form  $t+k$. 
The \emph{radius} of a rule is zero if its head is rigid, and 
it is the maximum difference between the offset of its head time argument and the
offset of a body time argument otherwise. 
A rule $r$ is \emph{forward-propagating} if it is Datalog, or it 
satisfies all of the following properties:
\begin{compactitem}
\item[--] it contains no time points;
\item[--] it has a single time variable, which occurs in the head;
\item[--] its radius is non-negative. 
\end{compactitem}
A query $Q$ is forward-propagating, or an \emph{fp-query} for short, if so is
each rule in $\Pi_Q$.  The radius of\/ $Q$ is the maximum radius amongst the
rules in $\Pi_Q$.  For $k \geq 0$, we denote as $Q^k$ the query
$\langle P_Q, \Pi_Q^k \rangle$ with $\Pi_Q^k$ the subset of rules in $\Pi_Q$
with radius at most $k$.

We denote the class of fp-queries as \fpclass{},
 and let 
    \ogclass{}, 
    \nrclass{}, and 
    \ognrclass{} be the subclasses of \fpclass{}
    where queries are required to be
    object-ground,
    non-recursive,
    and both object-ground and non-recursive, respectively.
\end{definition}
\begin{example}
  The query in our running Example \ref{ex:running} is
  forward-propagating. Its radius is two, which is justified
  by Rule \eqref{eq:attck}, where the offset of the head is two and
  the offset of the first body atom is zero.
\end{example}

The conditions in Definition \ref{def:fp-query} ensure that the derivation via rule application of a fact $\alpha$ holding at a time point $\tau$   
can be justified by facts holding  at time points no greater than $\tau$; as a result, one can safely
disregard all facts holding after $\tau$ for the purpose of deriving $\alpha$.  

The restrictions imposed by Definition \ref{def:fp-query} are
sufficient to ensure tractability of query evaluation, while at the
same time allowing for temporal recursion.  The following theorem
shows a stronger 
result, namely that query evaluation over fp-queries can be reduced to
query evaluation over standard non-temporal Datalog.

\begin{restatable}{theorem}{theoremevaluation}
    \label{theorem:evaluation}
    Let $\dclass$ be a class of finite datasets, let $\qclass \in \{ \fpclass, \nrclass, \ogclass, \ognrclass \}$,
    and let $\qclass'$ be the Datalog subset of $\qclass$.
    Then, $\evaluation_{\dclass}^{\qclass}$ is  \logspace{}-reducible to
    $\evaluation_{\dclass}^{\qclass'}$.
\end{restatable}
\begin{proof}[\proofidea]
    To check whether $\Pi_Q \cup D$ entails fact $\alpha$ holding at a time point $\tau$, it suffices to 
    consider facts (explicitly given or derived) holding at time points 
    in the interval between the minimum time point mentioned in $D$ and $\tau$; 
    such interval contains linearly-many time points due to $\tau$ being encoded in unary.
    We can then transform $\Pi_Q$ in $\logspace$  into a plain Datalog 
    program $\Pi'$  by first introducing an object for each time point in the interval, and then
    grounding the temporal arguments of all rules in $\Pi_Q$ over these objects.
    Clearly, it holds that $\Pi_Q \cup D$ entails $\alpha$ if so does $\Pi' \cup D$.
 \end{proof}

Theorem \ref{theorem:evaluation} allows us to immediately transfer known complexity bounds for query evaluation over
different classes of Datalog queries to the corresponding class of 
fp-queries---see, e.g., \cite{dantsin2001complexity,vorobyov1998complexity}.
In particular, it follows that evaluation of fp-queries is tractable in data complexity.

\begin{restatable}{corollary}{lemmaevaluation}
    \label{cor:evaluation}
    The following complexity bounds hold for the query evaluation problem over classes of fp-queries:
    \begin{compactitem}
    \item[--] $\evaluation^{\fpclass}$ is $\exptimecomplete$ and
    $\ptimecomplete$ in data; 
    \item[--] $\evaluation^{\nrclass}$ is $\pspacecomplete$ and in $\aczero$
    in data; and
    \item[--] $\evaluation^{\ogclass}$ is $\ptimecomplete$.
    \end{compactitem}
\end{restatable}


\section{A Generic Stream Reasoning Algorithm}
\label{sec:mbstreamreasoning}

\begin{algorithm}[t]
    \begin{footnotesize}
        \DontPrintSemicolon
        \SetKwRepeat{Loop}{loop}{end}
        \SetKwInput{Input}{Input}
         \SetKwInput{Parameters}{Parameters}
         \Parameters{Temporal fp-query $Q$, 
         window size $w$, and a subset $\Sigma$ of the IDBs in $Q$ with $P_Q \in \Sigma$. } 
         \Input{Background dataset $B$,
        stream $S$.}
        \BlankLine
        Assign $M:= B$ and $\tau:= 0$.\;
        \Loop{}{
            Receive $\restrict{S}{\tau}$ and
            assign $M:= M \cup \restrict{S}{\tau}$. \; 
            Add to $M$ all $\Sigma$-facts $\alpha$ holding at $\tau$ s.t.\ 
            $\Pi_Q \cup M \models \alpha$.  \;
            Stream out all  $P_Q$-facts in $\restrict{M}{\tau}$. \; 
            If $\tau \geq w$, remove from $M$  all facts in $\restrict{M}{\tau-w}$. \; 
            $\tau:=\tau+1$.\; 
    }
    \end{footnotesize}
       \caption{A generic stream reasoning algorithm} \label{alg:main}
\end{algorithm}

A stream reasoning algorithm receives  as input  
an unbounded  \emph{stream} $S$ of timestamped facts and a set $B$ of rigid \emph{background facts}, 
and outputs (also as a stream) the answers to a standing temporal query $Q$, which is considered fixed. 
Algorithm~\ref{alg:main}, which we describe next,  is a generic such algorithm 
that is applicable to any fp-query.
In the algorithm (as well as in the rest of the paper), we denote with 
$\restrict{F}{[\tau,\tau']}$ the subset of
temporal facts in a dataset $F$ holding in the interval $[\tau,\tau']$, and write
$\restrict{F}{\tau}$ for $\restrict{F}{[\tau,\tau]}$. Furthermore, from now on we will
silently assume all queries to be temporal.

Algorithm~\ref{alg:main} is parametrised by an fp-query $Q$, a non-negative integer
\emph{window size}  $w$ and a signature $\Sigma$, where the latter two parameters determine the set of facts $M$ 
kept in memory by the algorithm at any point in time.
The algorithm is initialised in Line 1, where the input set $B$ of rigid background facts is loaded into memory and the current time $\tau$ is set to zero.
The core of the algorithm is an infinite loop, where each iteration consists of the following four steps and the current time
$\tau$ is incremented at the end of
each iteration.
\begin{compactenum}[1.]
\item  The batch of
input stream facts holding at $\tau$ is received and loaded into memory (Line~3). 
\item All implicit facts over the relevant signature $\Sigma$ holding at $\tau$ are computed and materialised in memory (Line~4).
\item  Query answers holding at $\tau$ are read from memory and streamed out (Line~5);
\item  All facts (explicit in $S$ or implicitly derived) holding at $\tau-w$ are removed from memory (Line~6).
\end{compactenum}
In order to favour scalability, Algorithm \ref{alg:main} restricts at any point in time the set of facts
kept in memory  and therefore considered for query evaluation. This, however, carries the obvious risk that valid  answers holding over the entire stream
may be missed by the algorithm if the facts they depend on are removed from memory too early. Therefore, the window size of the algorithm
should be chosen so that the following correctness property is satisfied.

\begin{definition}
    A window size $w$ is \emph{valid}  for an fp-query $Q$, a signature $\Sigma$, and
    a class $\dclass$ of datasets if, when parametrised with $Q,w$ and $\Sigma$, and
    for each input $\langle B, S \rangle$ with $B \cup S \in \dclass$ and 
    each $n > 0$,
    the set of facts streamed out by Algorithm~\ref{alg:main} in the first $n$ iterations coincides with
    $\restrict{Q(B \cup S)}{[0,n-1]}$.
\end{definition}

In prior work \cite{ronca2017stream} we considered an algorithm
that does not keep derived facts (other than possibly query answers) in memory and thus only stores 
EDB facts from the input stream.
When applied to an fp-query $Q$,
the algorithm in our previous work 
can be seen as a variant of Algorithm \ref{alg:main} where $\Sigma = \{P_Q\}$. 
This variant of Algorithm~\ref{alg:main} is, however, problematic for recursive queries since
no valid window size may exist, in which case  the entire stream received so  far must be kept in memory to ensure correctness. 

\begin{proposition}
There exists no valid window size for the object-ground fp-query $Q$ where, 
for $A$ an EDB predicate,
$\Pi_Q  = \{ A(t) \rightarrow B(t);  B(t) \rightarrow B(t+1); B(t) \rightarrow
P_Q(t)\}$,
$\Sigma = \{P_Q\}$, 
and the class of all datasets.
\end{proposition}

To address this limitation, we focus from now onwards on a
\emph{full materialisation} variant of Algorithm  \ref{alg:main}, in which the signature parameter is fixed to
the set $\Sigma_Q$ of all IDB predicates in $Q$---that is, where the algorithm keeps in memory a complete materialisation of the query's program
for the relevant time points. Computing and incrementally maintaining a full materialisation is a common reasoning
approach adopted by many
rule-based systems \cite{mnph15incremental-BF-sameAs,mnpho14parallel-materialisation-RDFox,DBLP:journals/tocl/LeonePFEGPS06,DBLP:conf/ruleml/BagetLMRS15}.
In this setting, we will be able to ensure existence of a valid window size for any fp-query, 
and to show that a (maybe larger than needed) valid window size can be obtained syntactically by inspecting the rules in the query one at a time.

Towards this goal, we first analyse the aforementioned stream reasoning algorithm parametrised with query $Q$, 
window size $w$, and signature $\Sigma_Q$,
and show that only  the rules in $Q$
with radius at most $w$  can contribute to the  output.
\begin{restatable}{theorem}{thmoutput}
    \label{thm:output}
   Consider Algorithm  \ref{alg:main} parametrised with 
    $Q$, $w$ and $\Sigma_Q$.
    On input $\langle B, S \rangle$, 
    the set of $P_Q$-facts streamed out  in the first $n$ iterations
    coincides with $\restrict{Q^w(B \cup S)}{[0,n-1]}$.
\end{restatable}
\begin{proof}[\proofidea]
We show by induction on $\tau$
 that the set of temporal facts stored in $M$ right after executing Line~4 of the algorithm's main loop coincides with 
 the temporal facts entailed by 
 $\Pi_Q^w \cup B \cup S$ and holding at any $\tau' \in [\tau-w, \tau]$, which directly implies the statement of the theorem.
 On the one hand, we  show that any derivation from $\Pi_Q \cup M$ of a fact  $\alpha$ holding at $\tau$ can involve only rules from $\Pi_Q^w$; in particular,
 any derivation 
 involving a rule in $\Pi_Q$ with radius exceeding
 $w$ would require some fact holding 
  at a time point prior to $\tau-w$, where
  all such facts were removed from $M$ in previous iterations of the algorithm. 
On the other hand, we show that  all facts holding at $\tau$ entailed by $\Pi_Q^w \cup B \cup S$
admit a derivation involving only facts holding in $[\tau-w,\tau]$; by the induction hypothesis, all such facts
are in 
 $M$ when Line~4 of the algorithm is executed in the loop's iteration for $\tau$.
 \end{proof}
Theorem \ref{thm:output} immediately yields a characterisation of window size validity in terms
of query containment.

\begin{restatable}{corollary}{corvalidwindowcorrectness}
    \label{cor:valid-window-correctness}
     A window size $w$ is valid for an fp-query $Q$, the signature $\Sigma_Q$, and a class $\dclass$ of
    datasets iff $Q \sqsubseteq_{\dclass} Q^w$.
\end{restatable}

Since $Q$ and $Q^w$ coincide unless the radius of $Q$ exceeds $w$, we can conclude that the radius of $Q$
is always a valid window size.

\begin{restatable}{corollary}{corsyntacticcondition}
    \label{cor:syntactic-condition}
    Let $Q$ be an fp-query.
    Then, the radius of $Q$ is a valid window size for $Q$, $\Sigma_Q$, and any class of datasets~$\dclass$.
\end{restatable}


\section{The Window Validity Problem}
\label{sec:window}

The full materialisation of a query for any given time point
may be rather large. Having  a valid window size that is as small as possible
is thus important for Algorithm~\ref{alg:main} to be practically feasible, where even a small improvement on
the window size 
can lead to a significant reduction in the number of facts stored in memory and used for query evaluation.

In particular, the radius of the query yields a valid window size that may be larger than strictly necessary.
For instance, our running example query has a radius of two, which would require 
Algorithm~\ref{alg:main} to keep a full materialisation for three consecutive time points; 
however, the query admits a valid window size of just one since the policy implemented by Rule~\eqref{eq:attck} is subsumed by the other IDP in the example.

We next introduce the \emph{window validity problem}, which is to check whether a given window size is valid for a given
query. Due to Corollary \ref{cor:syntactic-condition}, 
computing a valid window of minimal size is clearly feasible using a logarithmic number of calls in the radius of the query
to an oracle 
for this problem. Furthermore, such minimal window can be computed ``offline'' before Algorithm~\ref{alg:main} 
is applied to any input data.

\begin{definition}
   Let $\qclass$ and $\dclass$ be classes of fp-queries and datasets, respectively.
   Then, $\textsc{Window}_{\dclass}^{\qclass}$ is the problem of deciding, 
   given $Q \in \qclass$ and $w \geq 0$ as input, whether $w$ is a valid window size 
   for $Q$, $\Sigma_Q$, and $\dclass$.
\end{definition}

Corollary \ref{cor:valid-window-correctness} provides a straightforward reduction
from our problem to query containment.
We next show that a reduction in the other direction also exists, which implies that
our problem has exactly the same complexity as query containment for all classes
of queries we consider.

\begin{restatable}{theorem}{thwindowcontainmentinterreducibleprev}
    \label{th:window-containment-interreducible-prev}
    $\window_{\dclass}^{\qclass}$ 
    and $\containment_{\dclass}^{\qclass}$ are interreducible in \logspace{}
    for each $\qclass \in \{ \fpclass, \ogclass, \nrclass, \ognrclass \}$ and
    each class $\dclass$ of datasets.
\end{restatable}
\begin{proof}[\proofidea]
   Consider queries $Q_1$ and $Q_2$ in $\qclass$, and assume w.l.o.g.\ that
   they do not share any IDBs other than the output predicate.
   In the case $\qclass \in \{ \ogclass, \ognrclass \}$ we also assume 
   w.l.o.g.\ that $Q_1$ and $Q_2$ are object-free.
   The key idea in reducing containment to window validity is to  
   merge $Q_1$ and $Q_2$ into a single query $Q$ such that
   \begin{compactenum}[1.]
   \item  both $Q_1$ and $Q_2$ may contribute to the answers of $Q$, and
   \item only $Q_2$ may contribute to the answers of $Q^w$ if $w$ is chosen as the
   maximum radius amongst $Q_1$ and $Q_2$.
   \end{compactenum}
    It follows that such $w$ is a valid window for $Q$, $\Sigma_Q$ and $\dclass$ iff 
    $Q_1 \sqsubseteq_{\dclass} Q_2$.
    To construct $\Pi_Q$, we first rename the output
    predicate in $\Pi_{Q_1}$ and $\Pi_{Q_2}$ to fresh $P_{Q_1}$ and $P_{Q_2}$, then union the
    resulting programs, and finally include the 
    following extra rules \eqref{eq:goal-rule-1} and \eqref{eq:goal-rule-2},
    where $A$ and $B$ are fresh unary temporal EDB predicates, 
    $w$ is as before, and $\vec{s} = \langle \vec{x}, t \rangle$ if $Q_1$ and $Q_2$
    are temporal and
    $\vec{s} = \vec{x}$ otherwise.
     \begin{align}
        \label{eq:goal-rule-1}
        A(t-w-1) \land B(t) \land P_{Q_1}(\vec{s}) &\to P_Q(\vec{x},t) \\
        \label{eq:goal-rule-2}
        B(t) \land P_{Q_2}(\vec{s}) &\to P_{Q}(\vec{x}, t)
    \end{align}
Note that both $Q_1$ and $Q_2$ contribute to the answers to $Q$ if the input stream contains
facts for $A$ and $B$ in all time points. Furthermore,  Rule \eqref{eq:goal-rule-1} has radius $w+1$; thus, it 
is not contained in $Q^w$ and cannot contribute to its answers.
\end{proof}

Since the language of fp-queries is an extension of
Datalog, 
it follows from Theorem \ref{th:window-containment-interreducible-prev} 
and standard results on Datalog query containment
that window validity is undecidable \cite{shmueli1993equivalence} in general and
 \conexptimehard{} for non-recursive queries  \cite{benedikt2010impact}.
 Furthermore, the results on  containment for non-recursive temporal queries in our
 prior work \cite{ronca2017stream} show that the aforementioned
 \conexptime{} lower bound is tight.
\begin{restatable}{corollary}{corknowncontainment}
    \label{cor:known-containment}
    Let $\dclass$ contain all finite datasets. Then,
    \begin{compactitem}
    \item[--] $\window_{\dclass}^{\qclass}$ is undecidable for any $\qclass$
    containing all Datalog queries, and
    \item[--] $\window_{\dclass}^{\nrclass}$ is \conexptimecomplete{}.
    \end{compactitem}
\end{restatable}

In the following section we show how to circumvent the undecidability result 
in Corollary \ref{cor:known-containment} while preserving the full power of
forward-propagating queries and, in particular, their ability to express temporal recursion.


\section{Window Validity for Fixed Object Domain}
\label{sec:boundedobject}

We consider the situation where the set of objects relevant to the application domain can be 
fixed in advance, in the sense that any input set of background facts and any input stream refer only to those objects.
This is a reasonable assumption in many applications of stream reasoning. For instance, when analysing temperature readings of
wind turbines, one may assume that the set of turbines generating the data
remains unchanged; furthermore, for the purpose of analysis
we can often also assume that temperature readings themselves can be discretised into relevant levels according to suitable thresholds.
In our running example, the set of nodes (pieces of data-generating computer equipment) present in the 
network is likely to change only rather rarely.

For the remainder of this section, let us fix a finite set $O$
of objects and let us denote with $\mathcal{O}$ the class of datasets mentioning objects from $O$ only.
Note that $\mathcal{O}$ is a valid class of datasets since it trivially satisfies 
the relevant assumptions  in Section \ref{sec:preliminaries}; thus, problems
$\bodwindow^{\qclass}$ and $\bodcontainment^{\qclass}$ are well-defined and, by Theorem \ref{th:window-containment-interreducible-prev}, 
they are 
also interreducible for
any class of queries  $\qclass$ mentioned in this paper.

In what follows, we show that $\bodwindow^{\qclass}$ is decidable and  establish  tight complexity bounds.

\subsection{Decidability and Upper Bounds}\label{sec:upper}

Fixing $O$ allows us to 
transform any input $Q$ to $\bodwindow^{\qclass}$ for $\qclass \subseteq
\fpclass$ into an object-ground query by grounding the object variables in $Q$
to constants in $O$; this yields an exponential reduction from $\bodwindow^{\qclass}$ to $\window^{\ogclass}$. 
Thus, our first step will be to
decide  window validity for object-ground queries, and for this we provide
a decision procedure for the corresponding query containment problem.

Let us consider fixed, but arbitrary, object-ground (temporal) queries $Q_1$ 
and $Q_2$ sharing an output predicate $G$. For simplicity, and without 
loss of generality, we assume that $Q_1$ and $Q_2$ contain no object
terms and hence all predicates in the queries are either nullary or unary and temporal.

We first show that there exists a number $b$ of exponential size in 
$\vert Q_1 \vert + \vert Q_2 \vert$ such that 
$Q_1 \not\sqsubseteq Q_2$
holds if and only if $G(\tau) \in Q_1(D)$ and $G(\tau) \notin Q_2(D)$ for some $\tau \in
[0,b]$ and some dataset $D$ over time points in $[0,b]$. 
We do so by constructing deterministic automata $\mathcal{A}_1$  and $\mathcal{A}_2$ for 
$Q_1$ and $Q_2$, respectively, and deriving $b$ from well-known bounds for the size of counter-examples to
automata containment.

\begin{restatable}{lemma}{lemmaautomatabound}
    \label{lemma:automata-bound}
    For each $i \in \{1,2\}$, let
    $\rho_i$ and $p_i$ be the radius and the size of the signature of $Q_i$, respectively.
    Let $b_i = 1 + 2^{p_i \cdot (\rho_i+2)}$,  and let $b = b_1 \cdot b_2$.
    
    If $Q_1 \not\sqsubseteq Q_2$, then there exists a time point $\tau \in [0,b]$ 
    and a dataset $D$ over time points in $[0,b]$ such that 
    $G(\tau) \in Q_1(D)$ and $G(\tau) \notin Q_2(D)$.
\end{restatable}
\begin{proof}[\proofidea]
We start with the observation that, given $Q_i$ and a dataset $D$, we can check whether the output predicate is derived  at any time point  
from $\Pi_{Q_i} \cup D$
using our generic stream reasoning algorithm. That is, we can start by loading
the rigid facts in $D$ and subsequently reading the temporal facts one time
point at a time while maintaining entailments over a window
of size $\rho_i$ until the output predicate is derived or $D$ does not contain any further time points.  

The correctness of this algorithm relies on the fact that $Q_i$ is forward-propagating and hence $\rho_i$ is a valid window. 
Based on this, we can construct a deterministic finite automaton $\mathcal{A}_i$  that captures $Q_i$ in the following sense:
on the one hand, each dataset $D$  corresponds to a word over the alphabet of the automaton, where the first symbol is
the set of rigid facts in $D$ and the remaining symbols encode the temporal
facts in $D$ one time point at a time on the other hand,
each state corresponds to a snapshot of the facts stored in memory by the algorithm, and a state is final if it corresponds to a snapshot in which the output predicate has
just been derived.   
Automaton $\mathcal{A}_i$ is defined as follows:

\begin{compactitem}
\item[--] A state is either the initial state $s_{\initstate}^i$, or a  
$(\rho_i+2)$-tuple where the first component is a subset of the rigid EDB 
predicates in $Q_i$, and the other components are subsets of the 
temporal (EDB and IDB) predicates in $Q_i$.
A  state is final if its last component contains the output predicate $G$.
\item[--] Each alphabet symbol  is a 
set $\Sigma$ of EDB predicates occurring in $Q_i$ such that
$\Sigma$ does not contain temporal and rigid predicates simultaneously.
\item[--] The transition function $\delta_i$ consists of
\begin{compactitem}
\item transitions  
$s_{\initstate}^i, \Sigma \mapsto \langle \Sigma, \emptyset, \dots, 
\emptyset \rangle$ such that $\Sigma$ consists of rigid predicates;
\item transitions
$\langle B, M_0, \dots, M_{\rho_i} \rangle, \Sigma \mapsto 
\langle B, M'_0, \dots, M'_{\rho_i} \rangle$ such that: $\Sigma$ consists of
temporal predicates; $M_j' = M_{j+1}$ for each $0 \leq j < \rho_i$; and
 $M_{\rho_i}'$ consists of each predicate $P$ satisfying   $\Pi_{Q_i} \cup B \cup H \cup U \models P(\rho_i)$ 
 for $H$ the set of all facts $R(j)$ with $R \in M_j$ and 
        $0 \leq j < \rho_i$,
        and $U$ the set of all facts $R(\rho_i)$ with $R \in \Sigma$.
\end{compactitem}
  \end{compactitem}

  The fact that each automaton $\mathcal{A}_i$ captures $Q_i$ in the sense
  described before ensures that the following properties immediately hold:
\begin{compactenum}[1.]
\item If $Q_1 \not\sqsubseteq Q_2$, then there exists a
  word 
  that is accepted by $\mathcal{A}_1$ and not by $\mathcal{A}_2$.
\item For each word of length $n$ 
  accepted by $\mathcal{A}_1$ and not by $\mathcal{A}_2$, there exists a
  dataset $D$ over time points in $[0,n-2]$ such that $G(n-2) \in Q_1(D)$ and
  $G(n-2) \notin Q_2(D)$.
\end{compactenum}
We finally argue that these properties imply the statement of the lemma.
If $Q_1 \not\sqsubseteq Q_2$ then, by Property~1,
there is a word accepted by $\mathcal{A}_1$ and not by $\mathcal{A}_2$.
By standard automata results, it follows that there is also a word 
accepted by $\mathcal{A}_1$ and not by $\mathcal{A}_2$ having length $n$ bounded by the product of the
number of states in $\mathcal{A}_1$ and $\mathcal{A}_2$, where the number of states in $\mathcal{A}_i$ is bounded by $b_i$. 
By Property~2, there exists a dataset $D$ over 
time points in $[0,n-2]$ such that $G(n-2) \in Q_1(D)$ and 
$G(n-2) \notin Q_2(D)$, where $n$ is bounded by $b$.
\end{proof}
Lemma \ref{lemma:automata-bound}  immediately suggests a non-deterministic algorithm for deciding
$Q_1 \not\sqsubseteq Q_2$, in which a witness dataset is
constructed and checked in each branch. In order to ensure that  the space used in each branch stays polynomial,
we exploit our observation in the beginning of the proof of Lemma \ref{lemma:automata-bound}.
A witness $D$ is 
guessed one time point at a time until reaching the bound $b$,  and $Q_1(D) \not\subseteq Q_2(D)$ is verified 
incrementally after each guess while keeping in memory just
a window of size bounded by the radiuses of $Q_1$ and $Q_2$.

\begin{restatable}{lemma}{semipropositionalcontainmentupperbound}
\label{lemma:semipropositional-containment-upper-bound}
    $\containment^{\ogclass}$ is in \pspace{}.
\end{restatable}

\begin{proof}
We decide $Q_1 \not\sqsubseteq Q_2$ using the following
algorithm, where $\rho$ is the maximum radius of $Q_1$ and $Q_2$.
\begin{compactenum}[1.]
    \item Guess a set $D_r$ of rigid facts and set $M_1$ and $M_2$ to $D_r$.
    \item  For each value of $\tau$ from $0$ to $b$ as in Lemma \ref{lemma:automata-bound}.
        \begin{compactenum}[a.]
            \item Guess $\restrict{D}{\tau}$.
            \item Set each  $M_i$ to $M_i \cup \restrict{D}{\tau}$.
            \item Add to each $M_i$  facts $\alpha$ at $\tau$ s.t.\
                $\Pi_{Q_i} \cup M_i \models \alpha$.
            \item If there is a $G$-fact in $\restrict{M_1}{\tau}$ and not in
                $\restrict{M_2}{\tau}$, accept.
            \item Remove from each $M_i$ all facts in $\restrict{M_i}{\tau-\rho}$.
        \end{compactenum}
    \item Reject.
\end{compactenum}

The algorithm correctly computes the answers over the guessed facts, since it
mimics Algorithm~1 and $\rho$ is a valid window for both queries.
By Lemma~\ref{lemma:automata-bound},  the algorithm finds a witness
dataset for non-containment whenever one exists.
Furthermore, the algorithm runs in polynomial space since the size of each $M_i$ is polynomial, and a polynomially-sized
counter suffices for checking the halting condition.
\end{proof}
Lemma \ref{lemma:semipropositional-containment-upper-bound}  yields
a $\pspace$ upper bound to $\window^{\ogclass}$.
In turn, it also provides an $\expspace$
upper bound to $\bodwindow^{\fpclass}$, which is obtained by first applying to
the input query $Q$ a grounding step where object variables from $\Pi_Q$ are
replaced with constants from the  object domain. 
Furthermore, this grounding process is polynomial
in the number of domain objects and exponential in the maximum number of object variables in
a rule from $\Pi_Q$; thus, the $\pspace$ upper bound in Lemma \ref{lemma:semipropositional-containment-upper-bound}
extends to any class of queries where the maximum number of object variables in
a rule can be bounded by a constant (which equals zero for $\ogclass$).

\begin{restatable}{theorem}{thmupperbounds}
\label{thm:upper-bounds}
The following upper bounds hold:
    \begin{compactitem}
    \item[--] $\bodwindow^{\fpclass}$ is in \expspace{};  and
    \item[--] $\bodwindow^{\qclass}$ it is in  \pspace{} for any class $\qclass$ 
        of fp-queries where the maximum number of object variables
        in any rule of any $Q \in \qclass$ is bounded by a constant.
    \end{compactitem}
\end{restatable}

By exploiting results from our prior work \cite{ronca2017stream}, we can show that
$\window^{\ognrclass}$ reduces
to query containment over non-recursive plain propositional Datalog. The latter 
can be decided in $\conptime$ by universally guessing a set of propositional symbols $D$ and then
 checking (in polynomial time) that $Q_2(D)$ holds whenever $Q_1(D)$ does, which yields a $\conptime$
bound for $\window^{\ognrclass}$ .
In turn, this bound yields a $\conexptime$ upper bound for $\bodwindow^{\nrclass}$ by means of an exponential grounding step of the object variables.
 Furthermore, such grounding is polynomial for any class $\qclass \subseteq \nrclass$  where 
 the maximum number of object variables in any rule is bounded by a constant; hence, the $\conptime$ upper bound for $\ognrclass$ seamlessly extends to any such class.

\begin{restatable}{theorem}{thmupperboundsnr}
\label{thm:upper-bounds-nr}
The following upper bounds hold:
    \begin{compactitem}
    \item[--] $\bodwindow^{\nrclass}$ is in \conexptime{}; and
    \item[--] $\bodwindow^{\qclass}$ is in \conptime{} for any class $\qclass \subseteq \nrclass$
    where the maximum number of object variables
    in any rule of any $Q \in \qclass$ is bounded by a constant.
    \end{compactitem}
\end{restatable}

\subsection{Lower Bounds} \label{sec:lower}

We next show that all the upper bounds established in Section \ref{sec:upper} are tight. 
We start by providing 
a matching $\pspace$ lower bound to $\window^{\ogclass}$.
\begin{restatable}{theorem}{semipropositionalcontainmentlowerbound}
\label{thm:semipropositional-containment-lower-bound}
    $\window^{\ogclass}$  is \pspacehard{}.
\end{restatable}
\begin{proof}[\proofidea]
    We show hardness for $\containment^{\ogclass}$, which implies the theorem's statement by
    Theorem \ref{th:window-containment-interreducible-prev}.
    The proof is by reduction from the
    containment problem for regular expressions. 
    Let $R_1$
    and $R_2$ be regular expressions over a common finite alphabet $\Sigma$.
    We construct object-free queries $Q_1$ and $Q_2$ with unary output temporal predicate $G$
    such that $R_1 \sqsubseteq R_2$ 
    if and only if $Q_1 \sqsubseteq Q_2$. 
    
    Each $Q_i$ is defined such that
    it captures 
    $R_i$ as described next.
    We encode  words in $\Sigma^*$ using facts over 
    unary temporal EDB predicates $F$ and $A_{\sigma}$
    for each alphabet symbol 
    $\sigma \in \Sigma$. Intuitively, a fact $F(\tau)$ indicates that
    $\tau$ is the first
    position of the word, whereas a fact $A_{\sigma}(\tau')$ with $\tau' \geq \tau$ means 
    that $\sigma$ is the symbol in position $\tau'-\tau$.
   Queries $Q_i$  are constructed from $R_i$ such that the following property $(\star)$ holds for each dataset
   $D$ over the aforementioned EDB predicates and each time point $\tau$:
   \begin{itemize}
   \item[ ] $(\star)$:   $G(\tau) \in Q_i(D)$ if and only if 
    there exists a word $\sigma_1 \ldots \sigma_n$ in the language of $R_i$  such that $D$ contains facts $F(\tau-n), A_{\sigma_1}(\tau-n), A_{\sigma_2}(\tau-n+1), \dots, 
    A_{\sigma_n}(\tau-1)$.
   \end{itemize}
   Property $(\star)$ implies the statement of the theorem.
   On the one hand, if $Q_1 \not\sqsubseteq Q_2$, then $G(\tau) \in Q_1(D)$ and $G(\tau) \not\in Q_2(D)$ for some $\tau$ and $D$; 
   by $(\star)$, the former implies existence of a word $s$ in $\mathcal{L}(R_1)$ such that $D$ contains the relevant facts, whereas the latter 
   together with the aforementioned property of $D$ implies that $s \not\in \mathcal{L}(R_2)$.
   On the other hand, $R_1 \not\sqsubseteq R_2$ implies that there exists $s = \sigma_1 \ldots \sigma_n$ with $s \in \mathcal{L}(R_1)$ and
   $s \not\in \mathcal{L}(R_2)$; let $D_s$ be the dataset consisting of facts 
   $$F(0), A_{\sigma_1}(0), A_{\sigma_2}(1), \dots, 
    A_{\sigma_n}(n-1)$$
 By $(\star)$, we then have $G(n) \in Q_1(D_s)$ and $G(n) \not\in Q_2(D_s)$, and hence $Q_1 \not\sqsubseteq Q_2$.

   We now define $Q_i = \langle G, \Pi_{R_i} \rangle$, 
   where $\Pi_{R_i}$ is defined inductively from $R_i$ as described next;
   note that, for $\Pi$ a program,
   we denote with $\Pi'$ (resp., $\Pi''$) the program obtained from $\Pi$ 
   by renaming each predicate $P$ not in $\{A_{\sigma}\mid \sigma \in \Sigma\}$ 
   to a globally fresh predicate $P'$ ($P''$) of the same arity.
   
   \begin{compactenum}[1.]
   \item   $R_i = \emptyset$.
    Then, $\Pi_{R_i}$ is the empty program.
   \item $R_i = \sigma$ for $\sigma \in \Sigma$.
    Then, $\Pi_{R_i}$ consists of rule
    \begin{equation*}
        F(t) \land A_{\sigma}(t) \to G(t+1).
    \end{equation*}
   \item   $R_i = \varepsilon$.
    Then, $\Pi_{R_i}$ consists of rule
    \begin{equation*}
        F(t) \to G(t).
    \end{equation*}
   \item  $R_i = S \cup T$.
    Then, $\Pi_{R_i}$ extends
    $\Pi_{S}' \cup \Pi_{T}''$ with rules
    \begin{align*}
        F(t) &\to F'(t) &  F(t) &\to F''(t)  \\
        G'(t) &\to G(t) & G''(t) &\to G(t).
    \end{align*}
   \item  $R_i = S \circ T$. Then,  
    $\Pi_{R_i}$ extends  $\Pi_S' \cup \Pi_T''$ with rules
    \begin{align*}
        F(t) &\to F'(t), & G'(t) &\to F''(t), & G''(t) &\to G(t).
    \end{align*}
    \item 
       $R_i = S^+$.
    Then, $\Pi_{R_i}$ extends $\Pi_S'$ with rules
    \begin{align*}
        F(t) &\to F'(t), & G'(t) &\to F'(t), &  G'(t) &\to G(t). 
    \end{align*}
   \end{compactenum}
  
   It can be checked using a simple induction that the construction ensures that $(\star)$ holds.
\end{proof}

Theorem \ref{thm:semipropositional-containment-lower-bound} implies
$\pspace$-hardness of $\bodwindow^{\qclass}$ for any class $\qclass$ of fp-queries 
 where the maximum number of object variables is bounded by a constant.
 
 We next show a matching $\expspace$ lower bound to the complexity of
 $\bodwindow^{\fpclass}$.
To this end, we upgrade the reduction in Theorem~\ref{thm:semipropositional-containment-lower-bound} 
to a reduction from the containment problem of succinct regular expressions---regular expression extended with an exponentiation
operation $R^k$ where $k$ is coded in binary \cite{sipser}.

\begin{restatable}{theorem}{bodcontainmentexpspacehard}
    \label{thm:bodcontainment-expspacehard}
    $\bodwindow^{\fpclass}$ is \expspacehard{}.
\end{restatable}
\begin{proof}[\proofidea]
   We show hardness of the corresponding query containment problem, which implies the statement 
   by Theorem \ref{th:window-containment-interreducible-prev}. Let $R_1$ and $R_2$ be succinct regular 
   expressions over the same vocabulary $\Sigma$. We construct 
   fp-queries $Q_1$ and $Q_2$ over the same unary temporal output predicate $G$ such that $R_1 \sqsubseteq R_2$ 
   if and only if $Q_1 \sqsubseteq Q_2$. 
   
   As in the proof of Theorem \ref{thm:semipropositional-containment-lower-bound}, we construct $Q_i$ such that it captures
   $R_i$. We encode words as before using unary temporal EDB predicates $F$ and $A_{\sigma}$ for each $\sigma \in \Sigma$.
   Also as before, we construct $Q_i$ from $R_i$ such that property $(\star)$ holds where $D$ in the formulation of $(\star)$ is over
   objects in $O$.
   
   We now define $Q_i = \langle G, \Pi_\mathrm{succ} \cup \Pi_{R_i} \rangle$, where
   $\Pi_{R_i}$ will be defined inductively over the structure of $R_i$, and
   $\Pi_\mathrm{succ}$ is a Datalog program that defines in the standard way
   \cite{dantsin2001complexity} rigid IDB successor predicates
   $\mathit{succ}^m$ 
   of arity $2m$ relating $m$-strings over objects $\bar{0}$ and $\bar{1}$ for
   each exponent $k$ occurring in $R_i$ with $m = \lceil \log_2 k \rceil$.  Now
   we proceed with the inductive definition of $\Pi_{R_i}$, which is analogous
   to that in the proof of Theorem
   \ref{thm:semipropositional-containment-lower-bound} with the following
   additional case, and the minor modification that successor predicates are
   never renamed apart:
   %
   \begin{compactenum}[1.] \setcounter{enumi}{6}
   \item $R_i = S^k$ for some succinct regular expression $S$ and $k \geq 2$.
       Then, $\Pi_{R_i}$ is
     constructed from $\Pi_S$ as follows. First, we replace each $n$-ary atom
     $P(\mathbf{p},s)$, for $\mathbf{p}$ a vector of object terms and $s$ a 
     temporal term, with $P'(\vec{p}, \vec{x}, s)$ for $P'$ a fresh
     predicate (unique to $P$) of arity $n+m$ with
     $m = \lceil \log_2 k \rceil$, and $\vec{x}$ a fixed $m$-vector of fresh
     object variables. Second, we extend the resulting program with the
     following rules, where $\vec{a}$ is the encoding of $k-1$ as a binary
     string over $\bar{0}$ and $\bar{1}$:
    \begin{align*}
        \textstyle
        F(t) &\to F'(\vec{\bar{0}}, t) \\
        \textstyle
        G'(\vec{a}, t) &\to G(t) \\
        \textstyle
        G'(\vec{x}, t) \land \mathit{succ}^m(\vec{x}, \vec{y})  
            &\to F'(\vec{y}, t)
    \end{align*}
   \end{compactenum}
   We can show inductively that $(\star)$ holds.
%
%
\end{proof}

To conclude, we turn our attention to the case of non-recursive queries.
A matching $\conptime$ lower bound to the complexity of 
$\window^{\ognrclass}$ is  obtained by a simple reduction from 
$3$-\textsc{Sat} to the complement of our problem. A matching $\conexptime$ lower bound
for $\bodwindow^{\nrclass}$ follows by a simple adaptation of the hardness proofs in \cite{benedikt2010impact}
for containment in non-recursive Datalog.
 
\begin{restatable}{theorem}{thmnrlowerbounds}
    \label{thm:nr-lower-bounds}
    $\window^{\ognrclass}$ is $\conptimehard$. 
    Furthermore, $\bodwindow^{\nrclass}$ is
    $\conexptimehard$ if $O$ has at least two objects.
\end{restatable}


\section{A Sufficient Condition for Window Validity}
\label{sec:sufficient}

The assumption that the object domain can be fixed in advance
may not be reasonable in some applications.
For instance, it may be the case that sensor values cannot be naturally
discretised into suitable levels according to a threshold, or that new sensors
are continuously activated on-the-fly.

As already established, dropping the fixed domain assumption leads to
undecidability of window validity for (recursive) fp-queries.  In this section,
we propose a sufficient condition for the validity of a window that can be
checked in exponential time without additional assumptions, and which leads to
smaller window sizes compared to the radius of the query.  Our condition relies
on the notion of \emph{uniform containment} of two programs $\Pi_1$ and
$\Pi_2$~\cite{DBLP:books/mk/minker88/Sagiv88}, which is sufficient to ensure
containment of any queries $Q_1$ and $Q_2$ based on $\Pi_1$ and $\Pi_2$,
respectively.

\begin{definition}
   An extended dataset $E$ is a (possibly infinite) set of (not necessarily EDB) facts.
   Program $\Pi_1$ is \emph{uniformly contained} in program $\Pi_2$,
    written $\Pi_1 \ucontained \Pi_2$, 
    if and only if,
    for each extended dataset $E$ and each fact $\alpha$,
    it holds that 
    $\Pi_1 \cup E \models \alpha$ implies 
    $\Pi_2 \cup E \models \alpha$.
   
    A window size $w$ is \emph{uniformly valid} for an fp-query $Q$ if and only if
    $\Pi_Q \ucontained \Pi_{Q^w}$. 
\end{definition}

It is straightforward to check that, given any queries $Q_1$ and $Q_2$, it
holds that $\Pi_{Q_1} \ucontained \Pi_{Q_2}$ implies $Q_1 \sqsubseteq
Q_2$. Hence, we can establish that uniform validity is a sufficient condition
for window validity, which is more precise than the syntactic condition given by
the radius.

\begin{restatable}{proposition}{propsufficient}
    \label{prop:sufficient}
Let $Q$ be an fp-query with radius $\rho$, and let $w$ be a non-negative integer.
If $w$ is a uniformly valid window size for $Q$, then $w$ is also a valid window size for $\Sigma_Q$ and
any class $\dclass$ of datasets. Furthermore, if $w$ is the smallest uniformly valid window size for $Q$, then
$w \leq \rho$.
\end{restatable}

\begin{example} \label{ex:uniform-1}
    Consider the query $Q$ where $\Pi_Q$ consists
    of the following rules and $A$ is the only EDB predicate:
    \begin{align*}
        A(t) &\to P_Q(t) &
        A(t-1) \land A(t) &\to P_Q(t)
    \end{align*}
    Query $Q$ has radius one. We can see that $w = 0$ is a (uniform) window.
    Intuitively, this is because the first rule entails the
    second; 
    thus, $\Pi_Q$ and $\Pi_{Q^w}$ are logically (and hence also uniformly)
    equivalent.
\end{example}


It is well-known that uniform program containment amounts to checking fact entailment~\cite{DBLP:books/mk/minker88/Sagiv88}. On the one hand,
to check $\Pi_1 \ucontained \Pi_2$, it suffices to show that $\Pi_2$ entails each rule $r$ in $\Pi_1$, which can in turn be 
checked by first ``freezing''  $r$ into an extended dataset $E$ for the body and a fact $\alpha$ for the head 
and then verifying whether $\Pi_2 \cup E \models \alpha$. On the other hand, to check whether $\Pi \cup E \models \alpha$,
it suffices to check uniform containment of a single rule $r$ in $\Pi$, where $r$ is
obtained from $E$ and $\alpha$ by replacing each constant with a fresh variable in the obvious way.

\begin{restatable}{theorem}{uniformcontainment}
  \label{thm:uniform-containment}
  Let 
  $\qclass \in \{ \fpclass, \nrclass, \ogclass, \ognrclass \}$
  and let $\mathcal{P}$ be the class of programs that occur in queries from
  $\mathcal{Q}$.  Then, uniform window validity over queries in $\mathcal{Q}$
  and fact entailment over programs in $\mathcal{P}$ are inter-reducible in
  $\logspace$.
\end{restatable}

The following complexity bounds for uniform window validity immediately follow from complexity results for fact entailment.

\begin{corollary}
Uniform window validity over a class $\mathcal{Q}$ of queries is
\begin{compactitem}
\item[--] $\exptimecomplete$ if $\mathcal{Q} = \fpclass$;
\item[--] $\pspacecomplete$ if $\mathcal{Q} = \nrclass$; 
\item[--] in $\ptime$ if $\mathcal{Q}$ is any subclass of\/ $\fpclass$
    where the maximum number of object variables in any rule of any $Q \in \qclass$ is bounded by a constant; and
\item[--] in $\aczero$ if $\mathcal{Q}$ is any subclass of\/ $\nrclass$ where
    the maximum number of object variables in any rule of any $Q \in \qclass$ is bounded by a constant.
\end{compactitem}
\end{corollary}

We see uniform validity as a reasonable compromise in practice. On the one hand, it may yield smaller window sizes than the radius of the query, thus reducing the
amount of information that a stream reasoning algorithm needs to retain in memory; on the other hand,
it can be checked while relying solely on query processing infrastructure, and hence without the need for specialised algorithms.


\section{Related Work}

The formal underpinnings of stream query processing in databases
 were established in \cite{babcock2002models,arasu2006cql}.
\citeauthor{arasu2006cql}~(\citeyear{arasu2006cql})
proposed CQL as an extension of
SQL with a window construct, which specifies the input data relevant for query processing at any point in time.
CQL has become since then the core of many other stream query languages, including languages 
for the Semantic Web \cite{barbieri2009csparql,barbieri2010incremental,le2011native,le2013elastic,dell2015towards}.

In the context of stream reasoning,
\citeauthor{zaniolo2012streamlog}~(\citeyear{zaniolo2012streamlog})
proposed Streamlog: a language  which extends  temporal Datalog with non-monotonic negation while at the same time 
restricting the syntax so that only facts over time points  mentioned in the data can be derived. 
LARS \cite{beck2015lars,beck2015answer,beck2016equivalent} is a temporal rule-based stream reasoning language featuring built-in window constructs
and negation interpreted according to the stable model semantics.
In contrast to temporal Datalog, the semantics of LARS assumes that 
the number of time points 
in a model is  a  part of the input to query evaluation, and hence is restricted to be finite.
Stream reasoning has also been  
considered in  ontology-based data access \cite{calbimonte2010enabling,ozcep2014stream}
as well as in the context of 
complex event processing \cite{anicic2011ep,dao2015enriching}.

There are have been  several proposals of Datalog extensions 
for reasoning over static temporal data. 
The language we consider  is a notational variant of
Datalog$_\mathrm{1S}$
\cite{chomicki1988temporal,chomicki1989relational,chomicki1990polynomial}.
Templog is an extension of Datalog with modal temporal operators \cite{abadi1989temporal};
\datalogmtl{} is an extension with metric temporal logic \cite{DBLP:conf/aaai/BrandtKKRXZ17}; and
the language proposed by   \citeauthor{toman1998datalog}~(\citeyear{toman1998datalog}) extends Datalog
with integer periodicity constraints. 

Our language of fp-queries is related to past temporal logic, where formulae are restricted to refer to
past time points only \cite{DBLP:books/daglib/0077033,chomicki1995encoding}.
 \citeauthor{chomicki1995encoding}~(\citeyear{chomicki1995encoding})   
 presents an incremental update algorithm for checking dynamic integrity constraints expressed in past temporal logic;
 similarly to our stream reasoning algorithm, \citeauthor{chomicki1995encoding}'s update algorithm 
 exploits the idea that the length of the stored history 
 throughout a sequence of updates can be bounded to a value depending only on the query.

The window validity problem was introduced in our prior work \cite{ronca2017stream} based on a generic
stream reasoning algorithm that only keeps EDB facts in memory.  We established undecidability for unrestricted queries, and
provided tight complexity bounds for the non-recursive case. Our current paper extends \cite{ronca2017stream} by generalising window validity to the case
where the underpinning stream reasoning algorithm can also keep IDB facts in memory; furthermore, we show decidability and tight complexity bounds
for recursive queries under the (rather mild) assumption  that the object domain can be fixed in advance.
The window validity problem is related to a problem considered in the context of database constraint checking by 
\citeauthor{chomicki1995encoding}~\shortcite{chomicki1995encoding}, 
who obtained positive results for queries in temporal 
first-order logic. It is also related to the \emph{forgetting} problem in logic programming \cite{DBLP:conf/aaai/WangSS05,DBLP:journals/ai/EiterW08}, where
the goal is to eliminate predicates while preserving certain logical consequences.


\section{Conclusion and Future Work}

\newcommand{\tcasegeneral}{General}
\newcommand{\tcasebo}{Object Dom.}
\newcommand{\tcaseog}{OG}
\newcommand{\tcasefp}{FP}
\newcommand{\tcasenr}{NR}
\newcommand{\tcellundecidable}{Undec.}
\newcommand{\shortcomplete}{-c}

We have studied the window validity problem in stream reasoning and  its computational properties for temporal Datalog. We showed that 
window validity is undecidable; however, decidability can be regained by
making mild assumptions on the input data.

We see many avenues for future work.
First, it would be interesting to consider window validity for
    extensions of temporal Datalog (e.g., with comparison atoms or stratified
    negation) as well as for \datalogmtl{}. 
Second, we have assumed throughout the paper that all numbers in input queries and data are coded in unary; it would be interesting to revisit our 
  technical results for the case where binary encoding is assumed instead.
Finally, our decidability results do not immediately yield implementable algorithms; we are planning to develop and implement  practical window validity checking 
algorithms under the fixed object domain assumption.

%
%


\section*{Acknowledgments}

This research was supported by 
the SIRIUS Centre for Scalable Data Access in the Oil and Gas Domain
and the EPSRC projects DBOnto, MaSI$^3$, and ED$^3$.

\bibliographystyle{aaai}
\bibliography{bibliography}

\ifextendedversion
    \appendix
    \newpage
\onecolumn

\section{Appendix}

In our proofs, we will make use of the following notion of derivations, which
is a variant of hyper-resolution derivations restricted to temporal Datalog.
\begin{definition}
    Let $\Pi$ be a program,
    let $F$ be a set of facts,
    and let $\alpha$ be a fact.
    A \emph{derivation} of $\alpha$ from $\Pi \cup F$ is a finite node-labelled
    tree such that:
    (i)~each node is labelled with a ground instance of a rule in $\Pi \cup F$;
    (ii)~fact $\alpha$ is the head of the rule labelling the root; and
    (iii)~for each node $v$, the body of the rule labelling $v$ contains an atom 
    $\alpha$ if and only if $\alpha$ is the head of the rule labelling a child 
    of $v$.
\end{definition}
By the completeness of hyper-resolution, it then follows that a temporal
Datalog program $\Pi$ entails a fact $\alpha$ from a set of facts $F$ if and
only if $\alpha$ has a derivation from $\Pi\cup F$.
\begin{proposition} \label{prop:entailment-derivations}
    Let $\Pi$ be a program,
    let $F$ be a set of facts,
    and let $\alpha$ be a fact.
    Then, $\Pi \cup F \models \alpha$ if and only if there exists a derivation 
    of $\alpha$ from $\Pi \cup F$.
\end{proposition}
In the rest, whenever a fact $\alpha$ is entailed by a program $\Pi$ and a set
of facts $F$, we directly assume the existence of a derivation of $\alpha$ from
$\Pi \cup F$, without referring to
Proposition~\ref{prop:entailment-derivations}.


\subsection{Proof of a Claim in the Preliminaries}

As promised in the preliminaries, we show the following claim.
\begin{claim}
    Let $\dclass$ be a class of datasets,
    and let $\dclass'$ be the class of all finite datasets in
    $\dclass$.
    Then, $Q_1 \sqsubseteq_{\dclass} Q_2$ iff
    $Q_1 \sqsubseteq_{\dclass'} Q_2$.
\end{claim}
\begin{proof}
    Trivially 
    $Q_1 \not\sqsubseteq_{\dclass'} Q_2$
    implies
    $Q_1 \not\sqsubseteq_{\dclass} Q_2$, since $\dclass' \subseteq \dclass$.
    For the converse, assume 
    $Q_1 \not\sqsubseteq_{\dclass} Q_2$.
    There exists a dataset $D \in \dclass$ and a fact $\alpha$ such that
    $\alpha \in Q_1(D)$ and $\alpha \notin Q_2(D)$.
    There exists a finite $D' \subseteq D$ such that $\alpha \in Q_1(D')$, since
    derivations are finite.
    By our assumption on the considered classes of datasets,
    there exists a finite dataset $D''$ with $D' \subseteq D'' \subseteq D$ and
    $D'' \in \dclass$; and hence $D'' \in \dclass'$.
    By monotonicity, it follows that
    $\alpha \in Q_1(D'')$ and $\alpha \notin Q_2(D'')$.
    Therefore, 
    $Q_1 \not\sqsubseteq_{\dclass'} Q_2$.
\end{proof}

\subsection{Proof of Theorem~\ref{theorem:evaluation}}

\begin{proposition}
\label{prop:relevant-time}
    Let $\Pi$ be a program consisting of forward-propagating rules, 
    let $F$ be a set of facts, and
    let $\alpha$ be a temporal fact having a derivation $\delta$ from 
    $\Pi \cup F$.
    Furthermore, 
    let $\tmin$ be the minimum time point in $F$, and 
    let $\tau$ be the time argument of $\alpha$.
    Then, each time point occurring in $\delta$ is in $[\tmin, \tau]$.
\end{proposition}
\begin{proof}
    Let $\delta$ be a derivation of $\alpha$ from $\Pi \cup F$.
    We prove the claim by induction on the height $n$ of $\delta$.
    
    In the base case $n=0$,
    and hence $\alpha$ is the only atom in $\delta$.
    Since no time point occurs in $\Pi$ by the properties of forward-propagating
    rules,
    it follows that $\tau$ occurs in $F$.
    Therefore $\tmin \leq \tau$ by the definition of $\tmin$,
    and hence trivially $\tau \in [\tmin, \tau]$.

    In the inductive case $n>0$, and we assume that each time point occurring in
    a derivation of a temporal fact $\beta$ from $\Pi \cup F$ of height at most
    $n-1$ is in $[\tmin, \tau']$ where $\tau'$ is the time argument of $\beta$.
    Let $r$ be the rule labelling the root of $\delta$,
    let $\beta$ be an atom in the body of $r$,
    and let $\delta'$ be a derivation of $\beta$ occurring as a subtree in 
    $\delta$.
    We have two cases.
    In the first case $\beta$ is rigid, and hence it is clear that each label of
    a node of $\delta'$ is an instance of a Datalog rule in $\Pi$ by the
    properties of forward-propagating rules, and hence no time point occurs in 
    $\delta'$.
    In the other case $\beta$ is temporal.
    Let $\tau'$ be the time argument of $\beta$.
    Note that $\tau' \leq \tau$ since $r$ is forward-propagating.
    It follows that each time point occurring in $\delta'$ is in $[\tmin,\tau]$
    by the inductive hypothesis.
\end{proof}

\theoremevaluation*
\begin{proof}
    We describe a $\logspace$-computable many-one reduction $\varphi$ from 
    $\evaluation^{\qclass}_{\dclass}$ to $\evaluation^{\qclass'}_{\dclass}$.
    An instance of $\evaluation^{\qclass}_{\dclass}$ is 
    $I = \langle Q, D, \alpha \rangle$ with $Q \in \qclass$, $D \in \dclass$, 
    and $\alpha$ a fact.
    We consider two cases, depending on whether $\alpha$ is rigid or temporal.

    Assume that $\alpha$ is rigid.
    Then, $\varphi$ maps $I$ to $\langle Q_1, D, \alpha \rangle$ where $Q_1$ is 
    $\langle P_Q, \Pi_1 \rangle$ with $\Pi_1$ the Datalog subprogram of $\Pi_Q$.
    We argue that $\alpha \in Q(D)$ iff $\alpha \in Q_1(D)$. 
    First, we have that $\alpha \in Q_1(D)$ implies $\alpha \in Q(D)$ by 
    monotonicity.
    Then, for the converse, assume $\alpha \in Q(D)$ and let $\delta$ be a
    derivation of $\alpha$ from $\Pi_Q \cup D$.
    Since $\alpha$ is rigid and $Q$ is an fp-query, the rule $r$ labelling the
    root of $\delta$ is an instance of a Datalog rule in $\Pi_Q$, and hence each
    atom in $r$ is rigid; 
    inductively the same holds for each label of a node of $\delta$.
    Therefore $\delta$ is a derivation of $\alpha$ from $\Pi_1 \cup D$,
    and hence $\alpha \in Q_1(D)$.

    Now, assume that $\alpha$ is temporal.
    We further split into two cases.

    In the first case $D$ contains no temporal fact,
    and we define $\varphi$ as mapping $I$ to $\langle Q_2, D, \alpha \rangle$
    with $Q_2 = \langle P_Q, \emptyset \rangle$.
    We have that $\alpha \in Q(D)$ iff $\alpha \in Q_2(D)$,
    since $\alpha \notin Q_2(D)$ holds trivially, and $\alpha \notin Q(D)$ holds 
    because $Q$ mentions no time point, by our assumption.

    In the other case, we have that $D$ contains a temporal fact.
    Let $\tmin$ be the minimum time point in $D$, and let $\tau_{\alpha}$ be the
    time argument of $\alpha$.
    Let $Q_3 = \langle P_Q, \Pi_3 \rangle$ with $\Pi_3$ the program consisting
    of each rule $r'$ obtained from a rule $r \in \Pi_Q$ by substituting the
    time variable in $r$---note that there is at most one time variable in $r$
    since $Q$ is an fp-query---so that each time argument in $r'$ is in the
    interval $[\tmin, \tau_{\alpha}]$.
    Since $Q_3$ is time-ground, it is clear that we can build a Datalog query 
    $Q_3'$ equivalent to $Q_3$ by replacing each atom $\beta$ in $Q_3$ with a
    rigid atom over a fresh predicate that is unique to the predicate and time
    argument of $\beta$.
    Then, we define $\varphi$ as mapping $I$ to 
    $\langle Q_3', D, \alpha \rangle$.
    We argue next that the reduction is correct.
    It suffices to show that $\alpha \in Q(D)$ iff $\alpha \in Q_3(D)$,
    since $Q_3$ and $Q_3'$ are equivalent.
    First, we have that $\alpha \in Q_3(D)$ implies $\alpha \in Q(D)$
    because each rule in $\Pi_3$ is an instance of a rule in $\Pi_Q$.
    Then, for the converse, assume 
    $\alpha \in Q(D)$ and let $\delta$ be a derivation of $\alpha$ from 
    $\Pi_Q \cup D$. We have that each time point occurring in $\delta$ is in 
    $[\tmin, \tau_{\alpha}]$ by Proposition~\ref{prop:relevant-time}, and hence 
    each label of a node of $\delta$ is an instance of a rule of
    $\Pi_3$. Therefore $\delta$ is a derivation of $\alpha$ from $\Pi_3 \cup D$,
    and hence $\alpha \in Q_3(D)$.

    We finally argue that $\varphi$ can be computed in logarithmic space.
    It is clear that we can check whether $\alpha$ is rigid or temporal, 
    check whether $D$ contains a temporal fact, compute the minimum
    time point in $D$ if one exists, compute renamings, etc... in logarithmic 
    space.  The critical step is computing $Q_3$. 
    This is doable in logarithmic space because it suffices to consider
    substitutions mapping time variables to the interval 
    $[\tmin - \rho, \tau_{\alpha} + \rho]$ with $\rho$ the radius of $Q$, and
    the former interval has linear size, since we have assumed that numbers in
    the input $I$ are coded in unary.
\end{proof}

\subsection{Proof of Theorem~\ref{thm:output}}

\begin{restatable}{proposition}{delayzero} 
    \label{proposition:delay-zero}
    Let $\Pi$ be a program consisting of forward-propagating rules, 
    let $F$ be a set of facts,
    and let $\alpha$ be a fact.
    Furthermore,
    let $\tau$ be the time argument of $\alpha$,
    and let $B$ be the rigid facts in $F$.
    If $\Pi \cup F \models \alpha$, then 
    $\Pi \cup B \cup \restrict{F}{[0,\tau]} \models \alpha$.
\end{restatable}
\begin{proof}
    If $\delta$ is a derivation of $\alpha$ from $\Pi \cup F$,
    then each time point in $\delta$ is at most $\tau$ by
    Proposition~\ref{prop:relevant-time},
    and hence $\delta$ is a derivation of $\alpha$ from 
    $\Pi \cup B \cup \restrict{F}{[0,\tau]}$.
\end{proof}

\begin{lemma}
    \label{lemma:window-subprogram-prev}
    Consider Algorithm  \ref{alg:main} parametrised with 
    $Q$, $w$ and $\Sigma_Q$. 
    On input $\langle B, S \rangle$, the set of temporal facts stored in $M$
    right after executing Line~\materlineref{} in any iteration of the main
    loop coincides with the set of temporal facts entailed by 
    $\Pi_Q^w \cup B \cup S$ and holding at any $\tau' \in [\tau-w, \tau]$.  
\end{lemma}
\begin{proof}
    Let $\langle B, S \rangle$ be an input to Algorithm~1.
    For each $n \geq 0$,
    let $M_n^{\inputlineref}$ and $M_n^{\materlineref}$ be the
    facts stored in $M$ by Algorithm~1 on input $\langle B, S \rangle$ right 
    after Lines~\inputlineref{}~and~\materlineref{}, respectively, in the 
    $(n+1)$-th iteration of the main loop;
    furthermore, note that $\tau$ has value $n$ in the $(n+1)$-th iteration of 
    the main loop.
    Then, consider the following observations.

    \observation{obs:base-mat} 
    $M_0^{\inputlineref}$ is $B \cup \restrict{S}{0}$.

    \observation{obs:other-mat}
    For each $n > 0$, $M_n^{\inputlineref}$ is 
    $B \cup \restrict{M_{n-1}^{\materlineref}}{[n-w,\infty)} \cup
    \restrict{S}{n}$.

    \smallskip\par
    Next, we show the two inclusions separately.

    \smallskip\par
    $(\subseteq)$
    We first show that each temporal fact stored by Algorithm~1 in $M$ in
    any iteration of the main loop right after executing Line~\materlineref{}
    is entailed by $\Pi_Q^w \cup B \cup S$ and has time argument in 
    $[\tau-w,\tau]$.
    It suffices to show that each $M_n^{\materlineref}$ is a subset of the facts
    entailed by $\Pi_Q^w \cup B \cup S$.
    We prove it by induction on $n \geq 0$.

    In the base case $n = 0$.    
    Let $\alpha$ be a fact in $M_0^{\materlineref}$.
    It is clear from the algorithm that
    (i)~$\alpha$ is in $M_0^{\inputlineref}$ or
    (ii)~$\alpha$ is a temporal fact with time argument zero such that 
    $\Pi_Q \cup M_0^{\inputlineref} \models \alpha$.
    In case~(i), we have that $\alpha \in B \cup \restrict{S}{0}$ by
    Observation~\ref{obs:base-mat},
    and hence in $\alpha \in B \cup S$. Therefore the claim holds by
    monotonicity.
    In case~(ii), we have that
    $\Pi_Q \cup B \cup \restrict{S}{0} \models \alpha$ by 
    Observation~\ref{obs:base-mat}.
    Let $\delta$ be a derivation of $\alpha$ from 
    $\Pi_Q \cup B \cup \restrict{S}{0}$.
    By Proposition~\ref{prop:relevant-time}, we have that zero is the only time
    point in $\delta$.
    Any instance of a rule with radius bigger than zero contains a time point
    different from zero, and hence $\delta$ does not contain such an instance.
    In particular, $\delta$ is a derivation of $\alpha$ from 
    $\Pi_Q^0 \cup B \cup \restrict{S}{0}$.
    Therefore $\Pi_Q^w \cup B \cup S \models \alpha$ by monotonicity.

    In the inductive case $n > 0$, and we assume that
    $\alpha \in M_{n-1}^{\materlineref}$ implies 
    $\Pi_Q^w \cup B \cup S \models \alpha$.
    Let $\alpha$ be a fact in $M_n^{\materlineref}$.
    It is clear from the algorithm that
    (iii)~$\alpha$ is in $M_n^{\inputlineref}$ or
    (iv)~$\alpha$ is a temporal fact with time argument $n$ such that 
    $\Pi_Q \cup M_n^{\inputlineref} \models \alpha$.
    We consider the two cases separately.

    In case~(iii), we have that  
    $\alpha \in B \cup \restrict{M_{n-1}^{\materlineref}}{[n-w,\infty)} \cup
    \restrict{S}{n}$ by Observation~\ref{obs:other-mat}.
    We have two subcases:
    if $\alpha \in B \cup \restrict{S}{n}$, then the claim holds by
    monotonicity; otherwise, we have that 
    $\alpha \in \restrict{M_{n-1}^{\materlineref}}{[n-w,\infty)}$,
    and hence the claim holds by the inductive hypothesis.

    In case~(iv), we have that
    $\Pi_Q \cup B \cup \restrict{M_{n-1}^{\materlineref}}{[n-w,\infty)} \cup
    \restrict{S}{n} \models \alpha$ 
    by Observation~\ref{obs:other-mat}.
    Let $F$ be the set of facts entailed by $\Pi_Q^w \cup B \cup S$.
    Note that $M_{n-1}^{\materlineref} \subseteq F$ by the inductive hypothesis.
    It follows that
    $\Pi_Q \cup B \cup \restrict{F}{[n-w,\infty)} 
    \cup \restrict{S}{n} \models \alpha$ by monotonicity.
    Let $\delta$ be a derivation of $\alpha$ from 
    $\Pi_Q \cup B \cup \restrict{F}{[n-w,\infty)} 
    \cup \restrict{S}{n}$.
    Again by Proposition~\ref{prop:relevant-time}, we have that each time point
    of $\delta$ is in $[n-w,n]$.
    Any instance of a rule with radius bigger than $w$ contains time points
    in an interval of size bigger than $w+1$, and hence $\delta$ does not
    contain such an instance.
    In particular, $\delta$ is a derivation of $\alpha$ from 
    $\Pi_Q^w \cup B \cup \restrict{F}{[n-w,\infty)} 
    \cup \restrict{S}{n}$.
    It follows that
    $\Pi_Q^w \cup B \cup \restrict{F}{[n-w,\infty)} 
    \cup \restrict{S}{n} \models \alpha$,
    hence 
    $\Pi_Q^w \cup B \cup F \cup S \models \alpha$ by monotonicity,
    and hence 
    $\Pi_Q^w \cup B \cup S \models \alpha$ since $F$ is entailed by 
    $\Pi_Q^w \cup B \cup S$.

    \smallskip\par
    $(\supseteq)$
    We now show that the set of temporal facts stored by Algorithm~1 in $M$
    in any iteration of the main loop right after executing Line~\materlineref{}
    contains each fact entailed by $\Pi_Q^w \cup B \cup S$ and having time
    argument in $[\tau-w,\tau]$.
    Let $\alpha$ be a temporal fact entailed by $\Pi_Q^w \cup B \cup S$ and 
    having time argument in $[n-w,n]$.
    It suffices to show that $\alpha \in M_n^{\materlineref}$ for every 
    $n \geq 0$.
    We prove it by induction on $n \geq 0$.
    
    In the base case $n = 0$, and hence $\alpha$ has a time argument in 
    $[-w,0]$---specifically, such a time argument is zero.
    By Proposition~\ref{proposition:delay-zero},
    we have that $\Pi_Q^w \cup B \cup \restrict{S}{0} \models \alpha$,
    hence $\Pi_Q^w \cup M_0^{\inputlineref} \models \alpha$ by
    Observation~\ref{obs:base-mat},
    hence $\Pi_Q \cup M_0^{\inputlineref} \models \alpha$ by monotonicity, and
    hence $\alpha \in M_0^{\materlineref}$ according to the algorithm.
    
    In the inductive case $n>0$, and we assume that
    $M_{n-1}^{\materlineref}$ contains each fact entailed by 
    $\Pi_Q^w \cup B \cup S$ and having a time argument in $[n-1-w, n-1]$.
    We consider two cases.
    In the first case we have that $\alpha$ has time argument in $[n-w,n-1]$,
    hence $\alpha \in M_{n-1}^{\materlineref}$ by the inductive hypothesis,
    hence $\alpha \in M_n^{\inputlineref}$ by Observation~\ref{obs:other-mat},
    and hence $\alpha \in M_n^{\materlineref}$ by the definition of the
    algorithm.
    In the second case we have that $\alpha$ has time argument $n$.
    According to the algorithm,
    it suffices to show that $\Pi_Q \cup M_n^{\inputlineref} \models \alpha$.
    We prove it by induction on the height $m$ of a derivation $\delta$ of 
    $\alpha$ from $\Pi_Q^w \cup B \cup S$.
    
    In the base case $m = 0$, and hence $\alpha$ is a fact in $B \cup S$.
    In particular, $\alpha \in \restrict{S}{n}$,  
    and hence $\alpha \in M_n^{\inputlineref}$ by
    Observation~\ref{obs:other-mat}.
    Therefore $\Pi_Q \cup M_n^{\inputlineref} \models \alpha$ by monotonicity.
    
    In the inductive case $m > 0$, and we assume that 
    $\Pi_Q \cup M_n^{\inputlineref} \models \beta$ for each fact $\beta$
    with time argument $n$ and having a derivation from 
    $\Pi_Q^w \cup B \cup S$ of height at most $m-1$.
    Let $r$ be the rule labelling the root of $\delta$, 
    let $\beta$ be an atom in the body of $r$,
    and let $\tau'$ be the time argument of $\beta$.
    First, $r$ is an instance of a rule of $\Pi_Q$, since 
    $\Pi_Q^w \subseteq \Pi_Q$ by definition.
    Second, we show that $\Pi_Q \cup M_n^{\inputlineref} \models \beta$.
    Note that $\tau' \in [n-w, n]$,
    since $r$ is forward-propagating and the radius of $r$ is at most $w$ by the
    definition of $\Pi_Q^w$.
    We have two cases:
    if $\tau' = n$, then $\Pi_Q \cup M_n^{\inputlineref} \models \beta$ by the
    `inner' inductive hypothesis;
    otherwise, we have that $\tau' \in [n-w,n-1]$, hence 
    $\beta \in M_{n-1}^{\materlineref}$ by the `outer' inductive hypothesis, 
    hence $\beta \in M_n^{\inputlineref}$ by Observation~\ref{obs:other-mat},
    and hence $\Pi_Q \cup M_n^{\inputlineref} \models \beta$ by monotonicity.
    The first and the second points imply 
    $\Pi_Q \cup M_n^{\inputlineref} \models \alpha$.
\end{proof}

\thmoutput*
\begin{proof}
    Consider the $n$-th iteration of Algorithm~1.
    The output of the algorithm is determined by Line~\outputlineref{}, where 
    the algorithm 
    outputs the $P_Q$-facts in the set $\restrict{M}{\tau}$, which consists of
    the facts entailed by $\Pi_Q^w \cup B \cup S$ having time argument 
    $\tau$,
    by Lemma~\ref{lemma:window-subprogram-prev}.
    The claim then follows from the fact that $\tau = n-1$.
\end{proof}

\subsection{Proof of Theorem~\ref{th:window-containment-interreducible-prev}}

\thwindowcontainmentinterreducibleprev*
\begin{proof}
    We show reducibility in the two directions separately.

    \smallskip\par\noindent
    $(\rightsquigarrow)$
    Consider an instance $I = \langle Q, w \rangle$ of 
    $\window_{\dclass}^{\qclass}$,
    and the function $\varphi$ mapping $I$ to the instance 
    $\langle Q, Q^w \rangle$ of $\containment_{\dclass}^{\qclass}$.
    Note that 
    (i)~$Q^w \in \qclass$ since removing any number of rules from $Q$ yields a
    query in $\qclass$,
    (ii)~$\varphi$ can clearly be computed in logarithmic space, and
    (iii)~$\varphi$ is a many-one reduction since $w$ is a valid window for 
    $Q$, $\Sigma_Q$ and $\dclass$ if and only if $Q \sqsubseteq_{\dclass} Q^w$,
    by Corollary~\ref{cor:valid-window-correctness}.

    \smallskip\par\noindent
    $(\leftsquigarrow)$
    Now we prove reducibility in the other direction.
    Let $\psi$ be the reduction given in the proof sketch of 
    Theorem~\ref{th:window-containment-interreducible-prev}.
    Let $I = \langle Q_1, Q_2 \rangle$ be an instance of
    $\containment_{\dclass}^{\qclass}$, and 
    let $\psi(I) = \langle Q, \rho \rangle$.
    It is clear that $\psi$ can be computed in logarithmic space.

    We argue next that $Q \in \qclass$.
    It is clear that $Q \in \fpclass$.
    If $Q_1, Q_2 \in \nrclass$, then $Q \in \nrclass$ because $Q_1$ and $Q_2$
    do not share IDB predicates by our assumption, and 
    rule~\eqref{eq:goal-rule-1} and rule~\eqref{eq:goal-rule-2} do not add
    cycles in the $\Pi_Q$-dependencies, since $P_Q$ does not occur in any body
    of a rule of $\Pi_Q$.
    If $Q_1, Q_2 \in \ogclass$, then $Q \in \ogclass$ since 
    rule~\eqref{eq:goal-rule-1} and rule~\eqref{eq:goal-rule-2} are object-free
    because we have assumed that $Q_1$ and $Q_2$ are object-free.
    The case where $Q_1, Q_2 \in \ognrclass$ follows from the two previous
    cases.

    We argue next that $\psi$ is a many-one reduction, by showing that
    $\rho$ is a valid window for $Q$, $\Sigma_Q$ and $\dclass$ if and only if
    $Q_1 \sqsubseteq_{\dclass} Q_2$.
    In the following, note that $Q^{\rho}$ is $Q$ after removing
    rule~\eqref{eq:goal-rule-1}.
    
    Assume that $\rho$ is a valid window for $Q$, $\Sigma_Q$ and $\dclass$,
    and hence $Q \sqsubseteq_{\dclass} Q^{\rho}$ 
    by Corollary~\ref{cor:valid-window-correctness}.
    We show that $Q_1 \sqsubseteq_{\dclass} Q_2$.
    Let $D$ be a dataset in $\dclass$, and let $\alpha$ be a fact in $Q_1(D)$.
    Let $\tau$ be the time argument of $\alpha$ if $\alpha$ is temporal,
    and let $\tau = \rho + 1$ otherwise.
    Let $D'$ be $D$ extended with $A(\tau-\rho-1)$ and $B(\tau)$.
    Note that $D' \in \dclass$ by our assumption on the considered classes of
    datasets.
    We have that $\alpha \in Q(D')$ by rule~\eqref{eq:goal-rule-1}.
    Since $Q \sqsubseteq_{\dclass} Q^{\rho}$ as argued above,
    it follows that $\alpha \in Q^{\rho}(D')$,
    and in particular $\alpha$ is derived by rule~\eqref{eq:goal-rule-2}.
    Therefore $\alpha \in Q_2(D)$ by the construction of $Q$.
    
    For the converse, assume that $Q_1 \sqsubseteq_{\dclass} Q_2$.
    We show that $\rho$ is a valid window for $Q$, $\Sigma_Q$ and $\dclass$.
    By Corollary~\ref{cor:valid-window-correctness},
    it suffices to show $Q \sqsubseteq_{\dclass} Q^{\rho}$.
    Let $D$ be a dataset in $\dclass$, and let $\alpha$ be a fact in $Q(D)$.
    We distinguish two cases.
    In the first case, 
    $\alpha$ is derived by rule~\eqref{eq:goal-rule-2},
    and hence $\alpha \in Q^{\rho}(D)$.
    In the other case, 
    $\alpha$ is derived by rule~\eqref{eq:goal-rule-1},
    hence $\alpha \in Q_1(D)$ by the construction of $Q$;
    furthermore we have that $B(\tau) \in D$, where $\tau$ is the time argument
    of $\alpha$ if $\alpha$ is temporal and just a time point otherwise.
    It follows that $\alpha \in Q_2(D)$ since 
    $Q_1 \sqsubseteq_{\dclass} Q_2$ by our assumption, and 
    hence $\alpha \in Q^{\rho}(D)$ by rule~\eqref{eq:goal-rule-2}.
\end{proof}

\subsection{Proof of Lemma~\ref{lemma:automata-bound}}
\label{asec:automata-bound}

We begin by restating the automaton construction given in the proof sketch of
Lemma~\ref{lemma:automata-bound}, and then formally state and prove its
correctness. Finally, we use the automata-theoretic characterisation to prove
Lemma~\ref{lemma:automata-bound}.

\subsubsection{Automaton Construction}

Let $Q$ be a temporal object-ground fp-query with output predicate $G$.
For simplicity, and without loss of generality, we assume that $Q$ contains no 
object terms and hence all predicates in the query are either
nullary or unary and temporal.
Furthermore, let $\rho$ be the radius of $Q$.
Then, the automaton $\mathcal{A}$ capturing $Q$ is as follows.
\begin{compactitem}
\item[--] A \emph{state} is either the initial state $s_{\initstate}$, or a  
    $(\rho+2)$-tuple where the first component is a subset of the rigid EDB 
    predicates in $Q$, and the other components are subsets of the 
    temporal (EDB and IDB) predicates in $Q$.
    A state is \emph{final} if its last component contains the output predicate $G$.
\item[--] Each \emph{alphabet symbol} is a 
    set $\Sigma$ of EDB predicates occurring in $Q$ such that
    $\Sigma$ does not contain temporal and rigid predicates simultaneously.
\item[--] The \emph{transition function} $\delta$ consists of
    \begin{compactitem}
    \item each transition
        $s_{\initstate}, \Sigma \mapsto \langle \Sigma, \emptyset, \dots, 
        \emptyset \rangle$ 
        such that $\Sigma$ consists of rigid predicates; and
    \item each transition
        $\langle B, M_0, \dots, M_{\rho} \rangle, \Sigma \mapsto 
        \langle B, M'_0, \dots, M'_{\rho} \rangle$ 
        such that: 
        (i)~$\Sigma$ consists of temporal predicates; 
        (ii)~$M_i' = M_{i+1}$ for each $0 \leq i < \rho$; and
        (iii)~$M_{\rho}'$ consists of each predicate $P$ satisfying  
        $\Pi_Q \cup B \cup H \cup U \models P(\rho)$ 
        for $H$ the set of all facts $R(i)$ with $R \in M_i$ and 
        $0 \leq i < \rho$,
        and $U$ the set of all facts $R(\rho)$ with $R \in \Sigma$.
    \end{compactitem}
\end{compactitem}

\subsubsection{Correctness of the Construction}

To argue the correctness of the construction, we first show the following auxiliary result.

\begin{proposition}
    \label{prop:syntactic-program-window}
    Let $\Pi$ be a program consisting of forward-propagating rules, 
    let $F$ be a set of facts,
    and let $\alpha$ be a temporal fact.
    Furthermore,
    let $\tau$ be the time argument of $\alpha$,
    let $\rho$ be the maximum radius of a rule in $\Pi$,
    let $B$ be the set of rigid facts in $F$,
    let $H$ be the set of temporal facts entailed by
    $\Pi \cup B \cup \restrict{F}{[0,\tau)}$ and having time argument in
    $[\tau-\rho,\tau)$,
    and let $U$ be $\restrict{F}{\tau}$.
    Then,
    $\Pi \cup F \models \alpha$ iff
    $\Pi \cup B \cup H \cup U \models \alpha$.
\end{proposition}
\begin{proof}
    We prove the two implications separately.

    \smallskip\par\noindent
    $(\Leftarrow)$
    Assume $\Pi \cup B \cup H \cup U \models \alpha$.
    Let $H'$ be the set of facts entailed by $\Pi \cup F$.
    Note that 
    $B \subseteq F$,
    $H \subseteq H'$,
    and $U \subseteq F$.
    It follows that
    $\Pi \cup F \cup H' \models \alpha$ by monotonicity,
    and hence $\Pi \cup F \models \alpha$ since 
    $\Pi \cup F \models H'$ by definition.

    \smallskip\par\noindent
    $(\Rightarrow)$
    Assume $\Pi \cup F \models \alpha$.
    Let $\delta$ be a derivation of $\alpha$ from $\Pi \cup F$.
    We prove the claim by induction on the height $n$ of $\delta$.
    In the base case $n=0$, and hence $\alpha \in F$.
    In particular $\alpha \in U$, and hence the claim holds by monotonicity.
    In the inductive case $n>0$, and 
    we assume that $\Pi \cup B \cup H \cup U \models \beta$ holds
    for each temporal fact $\beta$ with time argument $\tau$ having a derivation
    from $\Pi \cup F$ of height at most $n-1$.
    Let $r$ be the rule labelling the root of $\delta$,
    and let $\beta$ be an atom in the body of $r$.
    It suffices to show $\Pi \cup B \cup H \cup U \models \beta$.
    We distinguish two cases.
    In the first case $\beta$ is rigid, and hence $\Pi \cup B \models \beta$ 
    since it is clear that any derivation of any rigid fact such as $\beta$ from 
    $\Pi \cup F$ does not involve temporal facts,
    by the properties of forward-propagating rules; 
    the claim follows by monotonicity.
    In the other case $\beta$ is temporal.
    Let $\tau'$ be the time argument of $\beta$.
    Note that $\tau' \in [\tau - \rho, \tau]$ since 
    $r$ is forward-propagating and its radius is at most $\rho$.
    We distinguish again two cases.
    If $\tau' \in [\tau-\rho, \tau)$, then $\beta \in H$
    by Proposition~\ref{proposition:delay-zero} and the construction of $H$,
    and hence the claim holds by monotonicity.
    Otherwise, we have that $\tau'$ coincides with $\tau$, and hence
    the claim holds by the inductive hypothesis.
\end{proof}

Automaton $\mathcal{A}$ correctly captures $Q$ in the sense of the following
Claim~\ref{claim:query-cont-upper-correct-1} and
Claim~\ref{claim:query-cont-upper-correct-2}.

\begin{restatable}{claim}{querycontuppercorrectone}
\label{claim:query-cont-upper-correct-1}
    Let $D$ be a dataset, 
    and let $n$ be a non-negative integer.
    Furthermore,
    let $\rho$ be the radius of $Q$,
    let $w$ be the word $\langle \Sigma_0, \Sigma_1, \dots, \Sigma_n \rangle$ 
    where 
    $\Sigma_0$ is the set of rigid facts in $D$
    and each $\Sigma_i$ with $i > 0$ is the set $\{ A \mid A(i-1) \in D \}$,
    let $M_i = \emptyset$ for each $-\rho \leq i \leq 0$,
    let $M_i = \{ P \mid \Pi_Q \cup D \models P(i-1) \}$
    for each $1 \leq i \leq n$,
    and let $s_i = \langle \Sigma_0, M_{i-\rho}, \dots, M_i \rangle$ for
    each $0 \leq i \leq n$.
    Then, 
    $\langle s_{\initstate}, \Sigma_0, s_0, \Sigma_1, s_1, \dots,
    \Sigma_n, s_n \rangle$ is a run of $\mathcal{A}$.
\end{restatable}
\begin{proof}
    We prove the claim by induction on $n$.

    In the base case $n = 0$.
    We have that $\langle s_{\initstate}, \Sigma_0, s_0 \rangle$
    is a run of $\mathcal{A}$ by construction---note that 
    $s_0 = \langle \emptyset, \dots, \emptyset \rangle$.

    In the inductive case $n > 0$, and we assume that 
    $\langle s_{\initstate}, \Sigma_0, s_0, \dots, 
    \Sigma_{n-1}, s_{n-1} \rangle$ is a run of $\mathcal{A}$. 
    We have to show that     
    $\langle s_{\initstate}, \Sigma_0, s_0, \dots, 
    \Sigma_n, s_n \rangle$ is a run of $\mathcal{A}$, for which it suffices to 
    show that $\delta(s_{n-1}, \Sigma_n) = s_n$.

    Let $H$ be the set consisting of each fact $P(i)$ for $P \in M_{n-\rho+i}$ 
    and $0 \leq i < \rho$,
    and let $U$ be the set consisting of each fact $P(\rho)$ for 
    $P \in \Sigma_n$.
    Hence, according to the construction of $\delta$, it suffices to show that
    $M_n = \{ P \mid \Pi_Q \cup \Sigma_0 \cup H \cup U \models P(\rho) \}$.
    Now, we have that 
    $M_n = \{ P \mid \Pi_Q \cup D \models P(n-1) \}$ by construction.
    Let $H'$ be the set of temporal facts entailed by 
    $\Pi_Q \cup \Sigma_0 \cup \restrict{D}{[0,n-1)}$ and having time argument in 
    $[n-1-\rho,n-1)$.
    It follows that 
    $M_n = \{ P \mid \Pi_Q \cup \Sigma_0 \cup H' \cup \restrict{D}{n-1} 
    \models P(n-1) \}$ 
    by Proposition~\ref{prop:syntactic-program-window}.
    Let $H^\mathrm{s}$ and $U^\mathrm{s}$ be $H'$ and $\restrict{D}{n-1}$,
    respectively, after replacing each time point $\tau$ with 
    $\tau - n + 1 + \rho$. Note that the two former datasets are well-formed
    since each time point occurring in them is at least $n-1-\rho$.
    It follows that
    $M_n = \{ P \mid \Pi_Q \cup \Sigma_0 \cup H^\mathrm{s} \cup  U^\mathrm{s}
    \models P(\rho) \}$ since $Q$ is an fp-query---in particular, it mentions no
    time point.
    Finally, 
    $M_n = \{ P \mid \Pi_Q \cup \Sigma_0 \cup H \cup U \models P(\rho) \}$ 
    holds by Claim~\ref{claim:query-cont-upper-correct-1}.1 and
    Claim~\ref{claim:query-cont-upper-correct-1}.2, which are given next.

    \smallskip\noindent
    \emph{Claim~\ref{claim:query-cont-upper-correct-1}.1.} 
    It holds that $H = H^\mathrm{s}$.

    We first show $H \subseteq H^\mathrm{s}$.
    Let $\alpha \in H$.
    By the definition of $H$,
    we have that $\alpha$ is a temporal fact of the form $P(i)$ with 
    $P \in M_{n-\rho+i}$ and $0 \leq i < \rho$.
    It follows that
    $\Pi_Q \cup D \models P(n - \rho + i - 1)$ by the construction of 
    $M_{n-\rho+i}$,
    hence
    $\Pi_Q \cup \restrict{D}{[0,n-1)} \models P(n - \rho + i - 1)$ 
    by Proposition~\ref{proposition:delay-zero} and monotonicity,
    hence $P(n - \rho + i - 1) \in H'$ by the construction of $H'$,
    hence $P(i) \in H^\mathrm{s}$ by the construction of $H^\mathrm{s}$,
    and hence $\alpha \in H^\mathrm{s}$.

    We now show $H^\mathrm{s} \subseteq H$.
    Let $\alpha \in H^\mathrm{s}$.
    We have that $\alpha$ is a temporal fact of the form $P(i)$ such that
    $P(n-1-\rho+i) \in H'$ by the construction of $H^\mathrm{s}$.
    It follows that
    $\Pi_Q \cup \Sigma_0 \cup \restrict{D}{[0,n-1)} \models P(n-1-\rho+i)$ 
    by the construction of $H'$,
    hence $P \in M_{n-\rho+i}$ by the construction of $M_{n-\rho+i}$,
    hence $P(i) \in H$ by the construction of $H$,
    and hence $\alpha \in H$.

    This concludes the proof of Claim~\ref{claim:query-cont-upper-correct-1}.1.

    \smallskip\noindent
    \emph{Claim~\ref{claim:query-cont-upper-correct-1}.2.} 
    It holds that $U = U^\mathrm{s}$.

    We first show $U \subseteq U^\mathrm{s}$.
    Let $\alpha \in U$.
    We have that $\alpha$ is a temporal fact of the form $P(\rho)$ with 
    $P \in \Sigma_n$,
    hence $P(n-1) \in D$ by the construction of $\Sigma_n$,
    hence $P(\rho) \in U^\mathrm{s}$ by the construction of $U^\mathrm{s}$,
    and hence $\alpha \in U^\mathrm{s}$.

    We now show $U^\mathrm{s} \subseteq U$.
    Let $\alpha \in U^\mathrm{s}$.
    We have that $\alpha$ is a temporal fact of the form $P(i)$ such that
    $P(n-1) \in D$,
    hence $P \in \Sigma_n$ by the construction of $\Sigma_n$,
    hence $P(\rho) \in U$ by the construction of $U$,
    and hence $\alpha \in U$.

    This concludes the proof of Claim~\ref{claim:query-cont-upper-correct-1}.2,
    and hence the overall proof.
\end{proof}

\begin{restatable}{claim}{querycontuppercorrecttwo}
\label{claim:query-cont-upper-correct-2}
    Let $\langle \Sigma_0, \Sigma_1, \dots, \Sigma_n \rangle$ be a 
    word over the input alphabet with $n \geq 0$,
    and let $\rho$ be the radius of $Q$.
    Furthermore,
    let $D$ be the dataset 
    $\Sigma_0 \cup \{ A(i-1) \mid 1 \leq i \leq n \text{, and } A \in 
    \Sigma_i \}$,
    let $M_i = \emptyset$ for each $-\rho \leq i \leq 0$,
    and let $M_i = \{ P \mid \Pi_Q \cup D \models P(i-1) \}$ for each 
    $1 \leq i \leq n$.
    If $\langle s_{\initstate}, \Sigma_0, s_0, \dots, \Sigma_n, s_n \rangle$ 
    is a run of $\mathcal{A}$,
    then $s_i = \langle \Sigma_0, M_{i-\rho}, \dots, M_{i-1}, M_i \rangle$
    for each $0 \leq i \leq n$.
\end{restatable}
\begin{proof}
    Consider a run
    $\langle s_{\initstate}, \Sigma_0, s_0, \dots, \Sigma_n, s_n \rangle$ 
    of $\mathcal{A}$.
    We prove the claim by induction on $n$.

    In the base case $n=0$, hence 
    the considered run is $\langle s_{\initstate}, \Sigma_0, s_0 \rangle$,
    and hence the claim holds since 
    $\delta(s_{\initstate}, \Sigma_0) =  \langle \Sigma_0, \emptyset, \dots, 
    \emptyset \rangle$ by construction.

    In the inductive case $n > 0$, and we assume that the claim holds if we
    replace $n$ with $n-1$.
    In particular, the inductive hypothesis implies that
    $s_{n-1} = \langle \Sigma_0, M_{n-\rho-1}, \dots, M_{n-2}, M_{n-1} \rangle$.
    By the construction of $\delta$,
    we have that $s_n$ is of the form 
    $\langle \Sigma_0, M^0, \dots, M^{\rho-1}, M^{\rho} \rangle$
    with $M^0, \dots, M^{\rho-1}, M^{\rho}$ sets of temporal predicates.
    In order to prove the claim,
    it suffices to show that $M^{\rho-i} = M_{n-i}$ for each 
    $0 \leq i \leq \rho$.
    By the construction of $\delta$, we have the following.

    \smallskip\noindent
    \emph{Claim~\ref{claim:query-cont-upper-correct-2}.1.} 
    It holds that $M^{\rho-i} = M_{n-i}$ for each $0 < i \leq \rho$.
    \smallskip\par

    Hence, we are left to prove $M^{\rho} = M_n$.
    Let $H$ be the set consisting of each fact $P(i)$ for $P \in M^i$ and 
    $0 \leq i < \rho$,
    and let $U$ be the set consisting of each fact $P(\rho)$ for 
    $P \in \Sigma_n$.
    Hence, according to the construction of $\delta$,
    it suffices to show that
    $M_n = \{ P \mid \Pi_Q \cup \Sigma_0 \cup H \cup U \models P(\rho) \}$.
    Now, we have that $M_n = \{ P \mid \Pi_Q \cup D \models P(n-1) \}$ by
    construction.
    Let $H'$ be the set of temporal facts entailed by 
    $\Pi_Q \cup \Sigma_0 \cup \restrict{D}{[0,n-1)}$ and having time argument in 
    $[n-1-\rho,n-1)$.
    It follows that 
    $M_n = \{ P \mid \Pi_Q \cup \Sigma_0 \cup H' \cup \restrict{D}{n-1} 
    \models P(n-1) \}$ 
    by Proposition~\ref{prop:syntactic-program-window}.
    Let $H^\mathrm{s}$ and $U^\mathrm{s}$ be $H'$ and $\restrict{D}{n-1}$,
    respectively, after replacing each time point $\tau$ with 
    $\tau - n + 1 + \rho$. Note that the two former datasets are well-formed
    since each time point occurring in them is at least $n-1-\rho$.
    It follows that
    $M_n = \{ P \mid \Pi_Q \cup \Sigma_0 \cup H^\mathrm{s} \cup  U^\mathrm{s}
    \models P(\rho) \}$ since $Q$ is an fp-query---in particular, it mentions no
    time point.
    Finally, 
    $M_n = \{ P \mid \Pi_Q \cup \Sigma_0 \cup H \cup U \models P(\rho) \}$ 
    holds by Claim~\ref{claim:query-cont-upper-correct-2}.2 and
    Claim~\ref{claim:query-cont-upper-correct-2}.3, which are given next.

    \smallskip\noindent
    \emph{Claim~\ref{claim:query-cont-upper-correct-2}.2.} 
    It holds that $H = H^\mathrm{s}$.

    We first show $H \subseteq H^\mathrm{s}$.
    Let $\alpha \in H$.
    By the definition of $H$,
    we have that $\alpha$ is a temporal fact of the form $P(i)$ with 
    $P \in M^i$ and $0 \leq i < \rho$.
    Let $j = \rho-i$.
    It follows that $P \in M^{\rho-j}$,
    hence $P \in M_{n-j}$ by 
    Claim~\ref{claim:query-cont-upper-correct-2}.1,
    hence $P \in M_{n-\rho+i}$,
    hence $\Pi_Q \cup D \models P(n-\rho+i-1)$ by the construction of 
    $M_{n-\rho+i}$,
    hence $\Pi_Q \cup \restrict{D}{[0,n-1)} \models P(n-\rho+i-1)$
    by Proposition~\ref{proposition:delay-zero} and monotonicity,
    hence $P(n-\rho+i-1) \in H'$,
    hence $P(i) \in H^\mathrm{s}$,
    and hence $\alpha \in H^\mathrm{s}$.

    We now show $H^\mathrm{s} \subseteq H$.
    Let $\alpha \in H^\mathrm{s}$.
    We have that $\alpha$ is a temporal fact of the form $P(i)$ such that
    $P(n-1-\rho+i) \in H'$ by the construction of $H^\mathrm{s}$.
    It follows that
    $\Pi_Q \cup D \models P(n-1-\rho+i)$ by the construction of $H'$,
    hence $P \in M_{n-\rho+i}$ by the construction of $M_{n-\rho+i}$.
    Let $j = \rho-i$.
    It follows that $P \in M_{n-j}$,
    hence $P \in M^{\rho-j}$ 
    by Claim~\ref{claim:query-cont-upper-correct-2}.1,
    hence $P \in M^i$,
    hence $P(i) \in H$ by the construction of $H$,
    and hence $\alpha \in H$.

    This concludes the proof of Claim~\ref{claim:query-cont-upper-correct-2}.2.

    \smallskip\noindent
    \emph{Claim~\ref{claim:query-cont-upper-correct-2}.3.} 
    It holds that $U = U^\mathrm{s}$.

    We first show $U \subseteq U^\mathrm{s}$.
    Let $\alpha \in U$.
    We have that $\alpha$ is a temporal fact of the form $P(\rho)$ with 
    $P \in \Sigma_n$.
    It follows that $P(n-1) \in D$ by the construction of $D$,
    hence $P(\rho) \in U^\mathrm{s}$ by the construction of $U^\mathrm{s}$,
    and hence $\alpha \in U^\mathrm{s}$.

    We now show $U^\mathrm{s} \subseteq U$.
    Let $\alpha \in U^\mathrm{s}$.
    We have that $\alpha$ is a temporal fact of the form $P(i)$ such that
    $P(n-1) \in D$,
    hence $P \in \Sigma_n$ by the construction of $D$,
    hence $P(\rho) \in U$ by the construction of $U$,
    and hence $\alpha \in U$.

    This concludes the proof of Claim~\ref{claim:query-cont-upper-correct-2}.3,
    and hence the overall proof.
\end{proof}

\subsubsection{Proof of the Main Claim}

To show Lemma~\ref{lemma:automata-bound}, we first observe two properties of
our automata.

Let $Q_1$ and $Q_2$ be temporal object-ground fp-queries sharing an output 
predicate $G$.
For simplicity, and without loss of generality, we assume that $Q_1$ and $Q_2$ 
contain no object terms and hence all predicates in the queries are either
nullary or unary and temporal.
Furthermore, let $\mathcal{A}_1$ and $\mathcal{A}_2$ be the automata for $Q_1$
and $Q_2$, respectively, built as described in the previous section.

\begin{claim} \label{claim:automata-correctness-1}
    If $Q_1 \not\sqsubseteq Q_2$,
    then there exists a word of length at least $2$ that is accepted by
    $\mathcal{A}_1$ and not by $\mathcal{A}_2$.
\end{claim}
\begin{proof}
    Let $\tau$ be a time point and let $D$ be a dataset such that
    $G(\tau) \in Q_1(D)$ and $G(\tau) \notin Q_2(D)$.
    Furthermore, let $\rho_1$ be the radius of $Q_1$,
    and let $n = \tau + 1$.
    By Claim~\ref{claim:query-cont-upper-correct-1},
    there exists a run 
    $\langle s_{\initstate}^1, \Sigma_0, s_0, \dots, \Sigma_n, s_n \rangle$ 
    where $\Sigma_0$ is the set of rigid facts in $D$, 
    each $\Sigma_i$ with $i > 0$ is the set $\{ A \mid A(i-1) \in D \}$,
    and $s_n$ is of the form 
    $\varrho = \langle \Sigma_0, M_{n-\rho_1}, \dots, M_n \rangle$ 
    with $M_n = \{ P \mid \Pi_{Q_1} \cup D \models P(\tau) \}$.
    Since $G(\tau) \in Q_1(D)$, 
    we have that $G \in M_n$, 
    hence $s_n$ is final,
    hence $\varrho$ is an accepting run of $\mathcal{A}_1$,
    and hence $\mathcal{A}_1$ accepts the word 
    $w = \langle \Sigma_0, \dots, \Sigma_n \rangle$.
    Note that $w$ has length $n+1$, and hence at least $2$ as required.

    It suffices to show that $\mathcal{A}_2$ does not accept $w$.
    We prove it by contradiction, assuming that $\mathcal{A}_2$ accepts $w$.
    There exists an accepting run 
    $\varrho' = \langle s_{\initstate}^2, \Sigma_0,
    s_0', \dots, \Sigma_n, s_n' \rangle$.
    Let $\rho_2$ be the radius of $Q_2$,
    and let 
    $D' = \{ A(i-1) \mid 1 \leq i \leq n \text{, and } A \in \Sigma_i \}$.
    By Claim~\ref{claim:query-cont-upper-correct-2}, we have that
    $s_n'$ is of the form $\langle \Sigma_0, M_{n-\rho_2}', \dots, M_n' \rangle$
    with $M_n' = \{ P \mid \Pi_{Q_2} \cup \Sigma_0 \cup D' \models P(\tau) \}$.
    It follows that $G \in M_n'$ since $\varrho'$ is accepting,
    and hence $G(\tau) \in Q_2(\Sigma_0 \cup D')$.
    Therefore $G(\tau) \in Q_2(D)$ since
    $\Sigma_0 \cup D'  = D$, which contradicts our initial assumption.
\end{proof}

\begin{claim} \label{claim:automata-correctness-2}
    For each word $w$ of length $n \geq 2$ 
    accepted by $\mathcal{A}_1$ and not by $\mathcal{A}_2$,
    there exists a dataset $D$ over time points in $[0,n-2]$ such that
    $G(n-2) \in Q_1(D)$ and $G(n-2) \notin Q_2(D)$.
\end{claim}
\begin{proof}
    Let $w = \langle \Sigma_0, \dots, \Sigma_{n-1} \rangle$ be a word
    with $n \geq 2$.
    Assume that $\mathcal{A}_1$ accepts $w$ and $\mathcal{A}_2$ does not
    accept $w$.
    There exists an accepting run 
    $\varrho = \langle s_{\initstate}^1, \Sigma_0, s_0, \dots, 
    \Sigma_{n-1}, s_{n-1} \rangle$ 
    of $\mathcal{A}_1$.
    Let $D$ be the dataset 
    $\Sigma_0 \cup \{ A(i-1) \mid 1 \leq i \leq n-1 \text{, } 
    A \in \Sigma_i \}$.
    Note that $D$ is over time points in $[0,n-2]$ as required.
    Let $\rho_1$ be the radius of $Q_1$.
    By Claim~\ref{claim:query-cont-upper-correct-2}, 
    we have that 
    $s_{n-1} = \langle \Sigma_0, M^0, M^1, \dots, M^{\rho_1} \rangle$
    where $M^{\rho_1} = \{ P \mid \Pi_{Q_1} \cup D \models P(n-2) \}$.
    Since $\varrho$ is accepting, we have that $s_{n-1}$ is final,
    hence $G \in M^{\rho_1}$,
    and hence $G(n-2) \in Q_1(D)$.

    It suffices to show that $G(n-2) \notin Q_2(D)$.
    We prove it by contradiction, assuming that $G(n-2) \in Q_2(D)$.
    Let $\rho_2$ be the radius of $Q_2$,
    let $M_i = \emptyset$ for each $-\rho_2 \leq i \leq 0$,
    let $M_i = \{ P \mid \Pi_{Q_2} \cup D \models P(i-1) \}$ for each 
    $1 \leq i \leq n-1$,
    and let $s_i' = \langle B, M_{i-\rho_2}, \dots, M_i \rangle$
    for each $0 \leq i \leq n-1$.
    By Claim~\ref{claim:query-cont-upper-correct-1},
    we have that
    $\varrho' = \langle s_{\initstate}^2, \Sigma_0, s_0', 
    \dots, \Sigma_{n-1}, s_{n-1}' \rangle$ 
    is a run of $\mathcal{A}_2$.
    Since $G(n-2) \in Q_2(D)$ by our assumption,
    we have that $s_{n-1}'$ is final, and hence $\varrho'$ is an accepting run of 
    $\mathcal{A}_2$.
    Therefore $\mathcal{A}_2$ accepts $w$, which contradicts our initial 
    assumption.
\end{proof}

\lemmaautomatabound*
\begin{proof}
    Let $\tau$ be a time point and let $D$ be a dataset such that 
    $G(\tau) \in Q_1(D)$ and $G(\tau) \notin Q_2(D)$.
    By Claim~\ref{claim:automata-correctness-1},
    there exists a word of length at least $2$ accepted by $\mathcal{A}_1$ and 
    not by $\mathcal{A}_2$.
    Let $N_i$ be the number of states of $\mathcal{A}_i$;
    also note that $N_i \leq b_i$, since the set of 
    states of $\mathcal{A}_i$ consists of one initial state plus each
    $(\rho_i+2)$-tuple where each component is a subset of the predicates
    occurring in $Q_i$.
    By standard automata results, it follows that there exists a word of
    length $n$ with $2 \leq n \leq N_1 \cdot N_2 \leq b_1 \cdot b_2$ that is
    accepted by $\mathcal{A}_1$ and not by $\mathcal{A}_2$.
    By Claim~\ref{claim:automata-correctness-2},
    it follows that there exists a dataset $D'$ over time points in $[0,n-2]$ 
    such that $G(n-2) \in Q_1(D')$ and $G(n-2) \notin Q_2(D')$.
    Therefore $n-2$ and $D'$ are the desired time point and dataset, 
    respectively.
\end{proof}

\subsection{Proof of Theorem~\ref{thm:upper-bounds-nr}}

\begin{lemma} \label{lemma:ognr-containment-upper}
    $\containment^{\ognrclass}$ is in \conptime{}.
\end{lemma}
\begin{proof}
    First, note that $\containment^{\ognrclass}$ is 
    $\logspace$-reducible to $\containment^{\qclass}$, with 
    $\qclass$ the Datalog subset of $\ognrclass$, by the results in 
    \cite{ronca2017stream}.
    Therefore it suffices to show that the complement of the latter problem is 
    in \nptime{}.
    We give an algorithm, with input consisting of two object-ground
    non-recursive Datalog queries $Q_1$ and $Q_2$.
    The algorithm guesses a subset $D$ of the EDB atoms 
    occurring in $\Pi_{Q_1} \cup \Pi_{Q_2}$, 
    and then accepts if $Q_1(D) \not\subseteq Q_2(D)$.
    It is correct because it clearly suffices to consider datasets which are
    subsets of the EDB atoms occurring in $\Pi_{Q_1} \cup \Pi_{Q_2}$.
    It runs in polynomial time because each guessed dataset $D$ is of polynomial 
    size and evaluation of propositional queries is in \ptime{}---it amounts to 
    Horn satisfiability.
\end{proof}

\noindent
For the following theorem, note that $\mathcal{O}$ is the class of dataset over
objects from a given (but arbitrary) set of objects $O$.

\thmupperboundsnr*
\begin{proof}
    We start by noting that 
    $\window^{\ognrclass}$ is $\logspace$-reducible to 
    $\containment^{\ognrclass}$ 
    by Theorem~\ref{th:window-containment-interreducible-prev},
    which is in \conptime{} by Lemma~\ref{lemma:ognr-containment-upper}.
    Note also that, given any query $Q \in \nrclass$, we can ground its object 
    variables over $O$ in time asymptotically bounded by 
    $|\Pi_Q| \cdot |O + O^Q|^k$, where $O^Q$ is the set of objects occurring in 
    $Q$, and $k$ is the maximum number of object variables in a rule of $Q$; 
    such a grounding yields a query $Q' \in \ognrclass$ equivalent to $Q$.
    So, given an instance $I = \langle Q, w \rangle$ of 
    $\window^{\nrclass}_\mathcal{O}$, we can first map it to
    $I' = \langle Q', w \rangle$ with $Q'$ the object-grounding of $Q$, and then
    decide whether $\window^{\ognrclass}$ holds for $I'$ in
    $\conptime$, and hence whether $\window^{\nrclass}_\mathcal{O}$ holds for 
    $I$ in $\conexptime$, since $I'$ is exponential in $I$.
    For any class $\qclass \subseteq \nrclass$ where the maximum number of 
    object variables in any rule is bounded by a constant---i.e., $k$ in the
    expression $|\Pi_Q| \cdot |O + O^Q|^k$ can be considered fixed---we can 
    compute $I'$ in polynomial time, and hence we can decide whether 
    $\window^{\qclass}_\mathcal{O}$ holds for $I$ in $\conptime$.
\end{proof}


\subsection{Proof of Theorem~\ref{thm:semipropositional-containment-lower-bound}}

\semipropositionalcontainmentlowerbound*
\begin{proof}
    It suffices to show hardness for $\containment^{\ogclass}$
    which is $\logspace$-reducible to $\window^{\ogclass}$ by
    Theorem~\ref{th:window-containment-interreducible-prev}.
    Consider the reduction from containment of 
    succinct regular expressions to $\containment_\mathcal{O}^{\fpclass}$
    given in the following Section~\ref{sec:appendix-expspace-lower}.
    If the construction given there is restricted to (ordinary) regular
    expression, then it is easy to see that we obtain a reduction from 
    containment of regular expressions to $\containment^{\ogclass}$.
    The result then follows from the fact that containment of regular
    expressions is $\pspacecomplete$---see, e.g., \cite{sipser}.
\end{proof}

\subsection{Proof of Theorem~\ref{thm:bodcontainment-expspacehard}}
\label{sec:appendix-expspace-lower}

We first provide the full version of the query construction described in the
proof sketch of Theorem~\ref{thm:bodcontainment-expspacehard}.  We then show
the correctness of the construction and, finally, use it for proving
Theorem~\ref{thm:bodcontainment-expspacehard}.

\subsubsection{Query Construction}
Consider a succinct regular expressions (SRE) $R$ over a finite alphabet 
$\Sigma$---see, e.g., \cite{sipser} for the definition of SRE.
We build a query $Q_R = \langle G, \Pi_\mathrm{succ} \cup \Pi_R \rangle$ where
$\Pi_R$ will be defined inductively over the structure of $R$, and
$\Pi_\mathrm{succ}$ is a Datalog program that defines `successor' 
predicates $succ^m$.

Next, we define the program $\Pi_\mathrm{succ}$.
Let $\bar{0}$ and $\bar{1}$ be two fresh objects, intuitively standing for zero
and one respectively.
We use $\vec{x}$, $\vec{\bar{0}}$, and $\vec{\bar{1}}$ for denoting
tuples of fresh variables, $\bar{0}$'s, and $\bar{1}$'s, respectively.
We denote the length of a tuple $\vec{t}$ as $|\vec{t}|$.
Let $B$ be a fresh unary temporal IDB predicate,
and let $\mathit{succ}^m$ be a rigid IDB predicate of arity $2m$ for $m > 0$.
Let $\Pi_\mathrm{succ}^m$ for $m > 0$ be the program consisting of 
rule~\eqref{eq:binary-counting-1},
rule~\eqref{eq:binary-counting-2},
and each rule of the form~\eqref{eq:binary-counting-3} for $0 \leq i < m$ where
$|\vec{x}| = i$ and $|\vec{\bar{1}}| = |\vec{\bar{0}}| = m-i-1$.
\begin{align}
    \label{eq:binary-counting-1}
    &\to B(\bar{0}) \\
    \label{eq:binary-counting-2}
    &\to B(\bar{1}) \\
    \label{eq:binary-counting-3}
    \textstyle \bigwedge_{j=1}^i B(x_j) &\to 
    \mathit{succ}^m(\vec{x}, \bar{0}, \vec{\bar{1}}, 
    \vec{x}, \bar{1}, \vec{\bar{0}})
\end{align}
Each program $\Pi_\mathrm{succ}^m$ and its corresponding predicate 
$\mathit{succ}^m$ describe a finite successor relationship. 
Formally,
$\Pi_\mathrm{succ}^m \models \mathit{succ}^m(\vec{i}, \vec{j})$
holds if and only if 
(i)~$\vec{i}$ and $\vec{j}$ are $m$-tuples over 
$\{ \bar{0}, \bar{1} \}$, and
(ii)~$i + 1 = j$ for $i$ and $j$ the numbers encoded by 
$\vec{i}$ and $\vec{j}$, respectively.
Finally, $\Pi_\mathrm{succ}$ is the union of each $\Pi_\mathrm{succ}^m$ for
$m = \lceil \log_2 k \rceil$ and $k$ an exponent occurring in $R$.

Next we define the program $\Pi_R$.
Let $G$ be a fresh temporal unary IDB predicate,
let $F$ be a fresh temporal unary EDB predicate, and
let $A_{\sigma}$ be a fresh temporal unary EDB predicate for 
$\sigma \in \Sigma$.
For $\Pi$ a program,
we denote with $\phi(\Pi)$ and $\psi(\Pi)$ the programs obtained from $\Pi$ 
by renaming each predicate $P$ not in 
$\{A_{\sigma}\mid \sigma \in \Sigma\}$ and different from any $\mathit{succ}^m$ 
to globally fresh predicates $P^{\phi}$ and $P^{\psi}$, respectively, of 
the same arity as $P$---note that this renaming notation is different
from the one used in the proof sketch, which we believe to be more succinct but
less readable.
The program $\Pi_R$ is defined below, following the inductive definition of SRE.
We have three base cases where we define $\Pi_R$ from scratch, and four
inductive cases where we need to assume that we are given the programs for the
subexpressions of $R$.

\smallskip\par\noindent
\emph{Base case 1.}
It is the case where $R = \emptyset$.
Then, $\Pi_R$ is the empty program.

\smallskip\par\noindent
\emph{Base case 2.}
It is the case where $R = \sigma$ for $\sigma \in \Sigma$.
Then, $\Pi_R$ consists of the following rule.
\begin{equation}
    \label{eq:lower-bc2}
    F(t) \land A_{\sigma}(t) \to G(t+1)
\end{equation}

\smallskip\par\noindent
\emph{Base case 3.}
It is the case where $R = \varepsilon$.
Then, $\Pi_R$ consists of the following rule.
\begin{equation}
    \label{eq:lower-bc3}
    F(t) \to G(t)
\end{equation}

\smallskip\par\noindent
\emph{Inductive case 1.}
It is the case where $R = S \cup T$ for $S$ and $T$ SREs.
Then, $\Pi_R$ extends $\phi(\Pi_{S}) \cup \psi(\Pi_{T})$ with the following
rules.
\begin{align}
    \label{eq:lower-ic1-rule1}
    F(t) &\to F^{\phi}(t) \\
    \label{eq:lower-ic1-rule2}
    F(t) &\to F^{\psi}(t) \\
    \label{eq:lower-ic1-rule3}
    G^{\phi}(t) &\to G(t) \\
    \label{eq:lower-ic1-rule4}
    G^{\psi}(t) &\to G(t)
\end{align}

\smallskip\par\noindent
\emph{Inductive case 2.}
It is the case where $R = S \circ T$ for $S$ and $T$ SREs.
Then, $\Pi_R$ extends $\phi(\Pi_{S}) \cup \psi(\Pi_{T})$ with the following
rules.
\begin{align}
    \label{eq:lower-ic2-rule1}
    F(t) &\to F^{\phi}(t) \\
    \label{eq:lower-ic2-rule2}
    G^{\phi}(t) &\to F^{\psi}(t) \\
    \label{eq:lower-ic2-rule3}
    G^{\psi}(t) &\to G(t)
\end{align}

\smallskip\par\noindent
\emph{Inductive case 3.}
It is the case where $R = S^+$ for $S$ an SRE.
Then, $\Pi_R$ extends $\phi(\Pi_S)$ with the following rules.
\begin{align}
    \label{eq:lower-ic3-rule1}
    F(t) &\to F^{\phi}(t) \\
    \label{eq:lower-ic3-rule2}
    G^{\phi}(t) &\to F^{\phi}(t) \\ 
    \label{eq:lower-ic3-rule3}
    G^{\phi}(t) &\to G(t)
\end{align}

\smallskip\par\noindent
\emph{Inductive case 4.}
It is the case where $R = S^k$ for $S$ an SRE and $k \geq 2$.
Let $m = \lceil \log_2 k \rceil$---i.e., the number of bits to encode numbers in
the interval $[0,k-1]$.
Let $\vec{x}$ and $\vec{y}$ be $m$-tuples of fresh object variables,
and let $P'$ be a fresh temporal IDB predicate of arity $n+m$ for each temporal
(EDB or IDB) predicate $P$ of arity $n$.
Then, $\Pi_R$ is constructed from $\Pi_S$ as follows. 
First, we replace each atom $P(\mathbf{p},s)$ with 
$P'(\vec{p}, \vec{x}, s)$, where $\vec{p}$ is a vector of object terms and $s$
is a temporal term.
Second, we extend the resulting program with the following rules, where 
$\vec{a}$ is the encoding of $k-1$ as a binary string over 
$\bar{0}$ and $\bar{1}$.
\begin{align}
    \label{eq:lower-ic4-rule1}
    \textstyle
    F(t) &\to F'(\vec{\bar{0}}, t) \\
    \label{eq:lower-ic4-rule2}
    \textstyle
    G'(\vec{a}, t) &\to G(t) \\
    \label{eq:lower-ic4-rule3}
    \textstyle
    G'(\vec{x}, t) \land \mathit{succ}^m(\vec{x}, \vec{y})  
    &\to F'(\vec{y}, t)
\end{align}

\subsubsection{Correctness of the Construction}
Query $Q_R$, defined as above, correctly captures its corresponding SRE $R$
in the sense of the following Claim~\ref{claim:sreog-2} and
Claim~\ref{claim:sreog-1}.
Note that $\lang(R)$ denotes the language of an SRE $R$.

\newcounter{SreOgTwoCounter}
\newcommand{\sreogtwoclaim}[1]{%
    \smallskip\par\noindent%
    \refstepcounter{SreOgTwoCounter}\label{#1}%
    \emph{Claim~\ref{claim:sreog-2}.\arabic{SreOgTwoCounter}. }%
}
\newcommand{\sreogtwoclaimref}[1]{%
    Claim~\ref{claim:sreog-2}.\ref{#1}%
}

\begin{claim} \label{claim:sreog-2}
    Let $R$ be an SRE, and let $\Pi_R$ be the program for $R$.
    Furthermore, let $w = \langle \sigma_1, \dots, \sigma_n \rangle$ be a word 
    in $\lang(R)$,
    let $\tau$ be a time point,
    and let $D$ be a dataset.
    Assume that $F(\tau) \in D$,
    and $A_{\sigma_i}(\tau+i-1) \in D$ for each $1 \leq i \leq n$.
    Then, $\Pi_R \cup \Pi_\mathrm{succ} \cup D \models G(\tau+n)$.
\end{claim}
\begin{proof}
    We prove the claim by induction on the structure of $R$.

    In the base case, we have to prove that the claim holds for the three base
    cases in the inductive definition of $R$.

    \smallskip\par\noindent
    \emph{Base~case~1.}
    It is the case where $R = \emptyset$.
    This case cannot happen, since $\lang(R) = \emptyset$ contradicts our
    assumption that $w \in \lang(R)$.

    \smallskip\par\noindent
    \emph{Base~case~2.}
    It is the case where $R = \sigma$ for $\sigma \in \Sigma$.
    We have that $w = \langle \sigma \rangle$, and hence 
    $A_{\sigma}(\tau) \in D$.
    We have that $\Pi_R \cup D \models G(\tau+1)$ by
    rule~\eqref{eq:lower-bc2}.

    \smallskip\par\noindent
    \emph{Base~case~3.}
    It is the case where $R = \varepsilon$.
    We have that $w = \varepsilon$, and hence $n = 0$.
    We have that $\Pi_R \cup D \models G(\tau)$ by 
    rule~\eqref{eq:lower-bc3}.

    \smallskip\par
    In the inductive case, we have to prove that the claim holds in each of the
    following inductive cases, assuming that the claim holds for the
    subexpressions of $R$.

    \smallskip\par\noindent
    \emph{Inductive~case~1.}
    It is the case where $R = S \cup T$, for $S$ and $T$ SREs.
    Let $\Pi_S$ and $\Pi_T$ be the programs for $S$ and $T$, respectively.
    We have that $w \in \lang(S) \cup \lang(T)$.
    We consider two cases separately.
    In the first case we have that $w \in \lang(S)$,
    hence $\Pi_S \cup D \models G(\tau+n)$ by the inductive hypothesis,
    hence 
    $\phi(\Pi_S) \cup D \cup \{ F^{\phi}(\tau) \} \models G^{\phi}(\tau+n)$ 
    by the construction of $\phi(\Pi_S)$,
    and hence $\Pi_R \cup D \models G(\tau+n)$ 
    by rule~\eqref{eq:lower-ic1-rule1} and 
    rule~\eqref{eq:lower-ic1-rule3}.
    In the other case, symmetrically, we have that $w \in \lang(T)$,
    hence $\Pi_T \cup D \models G(\tau+n)$ by the inductive hypothesis,
    hence 
    $\psi(\Pi_T) \cup D \cup \{ F^{\psi}(\tau) \} \models G^{\psi}(\tau+n)$ 
    by the construction of $\psi(\Pi_T)$,
    and hence $\Pi_R \cup D \models G(\tau+n)$ 
    by rule~\eqref{eq:lower-ic1-rule2} and 
    rule~\eqref{eq:lower-ic1-rule4}.

    \smallskip\par\noindent
    \emph{Inductive~case~2.}
    It is the case where $R = S \circ T$, for $S$ and $T$ SREs.
    Let $\Pi_S$ and $\Pi_T$ be the programs for $S$ and $T$, respectively.
    We have that $w$ is of the form $w_1w_2$ with 
    $w_1 \in \lang(S)$ and $w_2 \in \lang(T)$.
    Let $w_1 = \langle \sigma_1^1, \dots, \sigma_{n_1}^1 \rangle$ 
    and $w_2 = \langle \sigma_1^2, \dots, \sigma_{n_2}^2 \rangle$.
    By the inductive hypothesis, 
    we have that $\Pi_S \cup D \models G(\tau+n_1)$,
    hence 
    $\phi(\Pi_S) \cup D \cup \{ F^{\phi}(\tau) \} \models G^{\phi}(\tau+n_1)$ 
    by the construction of $\phi(\Pi_S)$, and hence
    $\Pi_R \cup D \models F^{\psi}(\tau+n_1)$ by 
    rule~\eqref{eq:lower-ic2-rule1} and 
    rule~\eqref{eq:lower-ic2-rule2}.
    Again by the inductive hypothesis, we have that 
    $\Pi_T \cup D \cup \{ F(\tau + n_1) \} \models G(\tau + n_1 + n_2)$, hence
    $\psi(\Pi_T) \cup D \cup \{ F^{\psi}(\tau + n_1) \} \models 
    G^{\psi}(\tau + n_1 + n_2)$ 
    by the construction of $\psi(\Pi_T)$,
    and hence $\Pi_R \cup D \models G(\tau + n_1 + n_2)$ by
    rule~\eqref{eq:lower-ic2-rule3}.
    Therefore, $\Pi_R \cup D \models G(\tau + n_1 + n_2)$.

    \smallskip\par\noindent
    \emph{Inductive~case~3.}
    It is the case where $R = S^+$, for $S$ an SRE.
    Let $\Pi_S$ be the program for $S$.
    We have that $w$ is of the form $w_1w_2 \dots w_k$ with $k > 0$ and 
    $w_1,w_2, \dots, w_k \in \lang(S)$.
    Let $w_i = \langle \sigma_1^i, \dots, \sigma_{n_i}^i \rangle$
    for each $1 \leq i \leq k$.
    Note that the length of $w$ is $N = \sum_{i=1}^k n_i$.
    The following claim implies that
    $\Pi_R \cup D \models G^{\phi}(\tau + N)$,
    and hence $\Pi_R \cup D \models G(\tau + N)$ by
    rule~\eqref{eq:lower-ic3-rule3}.

    \sreogtwoclaim{claim:sreogtwo-1}
    For each $1 \leq i \leq k$, it holds that 
    $\Pi_R \cup D \models G^{\phi}(\tau + \sum_{j=1}^i n_j)$.

    We prove \sreogtwoclaimref{claim:sreogtwo-1} by induction on $i$ from $1$ 
    to $k$.
    In the base case $i = 1$.
    By the `outer' inductive hypothesis we have that 
    $\Pi_S \cup D \models G(\tau+n_1)$,
    hence 
    $\phi(\Pi_S) \cup D \cup \{ F^{\phi}(\tau) \} \models G^{\phi}(\tau+n_1)$
    by the construction of $\phi(\Pi_S)$,
    and hence $\Pi_R \cup D \models G^{\phi}(\tau+n_1)$ 
    by rule~\eqref{eq:lower-ic3-rule1}.
    In the inductive case $i > 1$, and we assume that
    $\Pi_R \cup D \models G^{\phi}(\tau + \sum_{\ell=1}^j n_{\ell})$
    for each $1 \leq j \leq i-1$.
    By the `outer' inductive hypothesis we have that
    $\Pi_S \cup D \cup \{ F(\tau + \sum_{\ell=1}^{i-1} n_{\ell}) \} \models
    G(\tau + \sum_{\ell=1}^i n_{\ell})$,
    hence 
    $\phi(\Pi_S) \cup D \cup 
        \{ F^{\phi}(\tau + \sum_{\ell=1}^{i-1} n_{\ell}) \} \models 
        G^{\phi}(\tau + \sum_{\ell=1}^i n_{\ell})$
    by the construction of $\phi(\Pi_S)$,
    and hence
    $\Pi_R \cup D \models G^{\phi}(\tau + \sum_{\ell=1}^i n_{\ell})$
    since 
    $\Pi_R \cup D \models F^{\phi}(\tau + \sum_{\ell=1}^{i-1} n_{\ell})$ by
    the `inner' inductive hypothesis and by 
    rule~\eqref{eq:lower-ic3-rule2}.

    This concludes the proof of \sreogtwoclaimref{claim:sreogtwo-1}.
    \smallskip\par

    \smallskip\par\noindent
    \emph{Inductive~case~4.}
    It is the case where $R = S^k$, for $k \geq 2$ and $S$ an SRE.
    Let $\Pi_S$ be the program for $S$.
    We have that $w$ is of the form $w_1w_2 \dots w_k$ with
    $w_1,w_2, \dots, w_k \in \lang(S)$.
    Let $w_i = \langle \sigma_1^i, \dots, \sigma_{n_i}^i \rangle$
    for each $1 \leq i \leq k$.
    The following claim implies that 
    $\Pi_R \cup \Pi_\mathrm{succ} \cup D \models 
    G'(\vec{b}, \tau + \sum_{i=1}^k n_i)$
    where $\vec{b}$ is the binary encoding of $k-1$,
    and hence
    $\Pi_R \cup \Pi_\mathrm{succ} \cup D \models G(\tau + \sum_{i=1}^k n_i)$ by 
    rule~\eqref{eq:lower-ic4-rule2}.

    \sreogtwoclaim{claim:sreogtwo-2}
    For each $1 \leq i \leq k$, it holds that 
    $\Pi_R \cup \Pi_\mathrm{succ} \cup D 
    \models G'(\vec{b}, \tau + \sum_{j=1}^i n_j)$
    where $\vec{b}$ is the binary encoding of $i-1$.

    We prove \sreogtwoclaimref{claim:sreogtwo-2} by induction on $i$ from 
    $1$ to $k$.
    In the base case $i = 1$.
    We have that 
    $\Pi_R \cup \Pi_\mathrm{succ} \cup D \models F'(\vec{\bar{0}}, \tau)$
    by rule~\eqref{eq:lower-ic4-rule1}.
    It follows that 
    $\Pi_R \cup \Pi_\mathrm{succ} \cup D \models G'(\vec{\bar{0}}, \tau + n_1)$,
    because $\Pi_S \cup D \models G(\tau + n_1)$ by the `outer' inductive
    hypothesis, and by the construction of $\Pi_R$.
    In the inductive case $i > 1$, and we assume that 
    for each $1 \leq j \leq i-1$, it holds that 
    $\Pi_R \cup \Pi_\mathrm{succ} \cup D \models 
    G'(\vec{b}, \tau + \sum_{\ell=1}^j n_{\ell})$
    where $\vec{b}$ is the binary encoding of $j-1$.
    In particular,
    $\Pi_R \cup \Pi_\mathrm{succ} \cup D \models 
    G'(\vec{b}, \tau + \sum_{\ell=1}^{i-1} n_{\ell})$
    where $\vec{b}$ is the binary encoding of $i-2$,
    and hence
    $\Pi_R \cup \Pi_\mathrm{succ} \cup D \models 
    F'(\vec{c}, \tau + \sum_{\ell=1}^{i-1} n_{\ell})$
    where $\vec{c}$ encodes $i-1$,
    by rule~\eqref{eq:lower-ic4-rule3} and by the construction of
    $\Pi_\mathrm{succ}$.
    It follows that
    $\Pi_R \cup \Pi_\mathrm{succ} \cup D \models 
    G'(\vec{c}, \tau + \sum_{\ell=1}^i n_{\ell})$,
    because 
    $\Pi_S \cup D \cup \{ F(\tau + \sum_{\ell=1}^{i-1} n_{\ell}) \} 
    \models G(\tau + \tau + \sum_{\ell=1}^i n_{\ell})$ 
    by the `outer' inductive hypothesis, and by the construction of
    $\Pi_R$.

    This concludes the proof of \sreogtwoclaimref{claim:sreogtwo-2},
    and hence the overall proof.
\end{proof}

\newcounter{SreOgOneCounter}
\newcommand{\sreogoneclaim}[1]{%
    \smallskip\par\noindent%
    \refstepcounter{SreOgOneCounter}\label{#1}%
    \emph{Claim~\ref{claim:sreog-1}.\arabic{SreOgOneCounter}. }%
}
\newcommand{\sreogoneclaimref}[1]{%
    Claim~\ref{claim:sreog-1}.\ref{#1}%
}

\begin{claim}
    \label{claim:sreog-1}
    Let $R$ be an SRE, and let $\Pi_R$ be the program for $R$.
    Furthermore, let $D$ be a dataset.
    Assume that $\Pi_R \cup \Pi_\mathrm{succ} \cup D \models G(\tau)$.
    Then, 
    there exists a word 
    $w = \langle \sigma_1, \dots, \sigma_n \rangle \in \lang(R)$
    such that
    $F(\tau-n) \in D$ and
    $A_{\sigma_i}(\tau-n+i-1) \in D$ for each $1 \leq i \leq n$.
\end{claim}
\begin{proof}
    We prove the claim by induction on the structure of $R$.
    Note that $\Pi_\mathrm{succ}$ needs to be considered in Inductive~case~4
    only, since the other cases mention no predicate of the form
    $\mathit{succ}^m$.

    In the base case, we have to prove that the claim holds for the three base
    in the inductive definition of $R$.

    \smallskip\par\noindent
    \emph{Base~case~1.}
    It is the case where $R = \emptyset$.
    We show that this case cannot happen.
    We would have that $\Pi_R = \emptyset$ by construction,
    and hence $\Pi_R \cup D \not\models G(\tau)$, 
    which contradicts our initial assumption.

    \smallskip\par\noindent
    \emph{Base~case~2.}
    It is the case where $R = \sigma$ for $\sigma \in \Sigma$.
    The word $\langle \sigma \rangle$ is as required since 
    since $F(\tau-1) \in D$ and $A_{\sigma}(\tau-1) \in D$ 
    by the construction of $\Pi_R$.

    \smallskip\par\noindent
    \emph{Base~case~3.}
    It is the case where $R = \varepsilon$.
    The empty word is as required since
    since $F(\tau) \in D$ 
    by the construction of $\Pi_R$.

    \smallskip
    In each of the following inductive cases, we have to prove that the claim 
    holds for $R$ assuming that the claim holds for the subexpressions of $R$.

    \smallskip\par\noindent
    \emph{Inductive~case~1.}
    It is the case where $R = S \cup T$, for $S$ and $T$ SREs.
    Let $\Pi_S$ and $\Pi_T$ be the programs for $S$ and $T$, respectively.
    We have that $\Pi_R \cup D \models G^{\phi}(\tau)$ or
    $\Pi_R \cup D \models G^{\psi}(\tau)$.
    Let $D'$ be the dataset consisting of each $A_{\sigma}$-fact in $D$ with 
    $\sigma \in \Sigma$.
    It is clear from the construction of $\Pi_R$ that the
    following claim holds.

    \sreogoneclaim{claim:sre-og-0}
    One of the following holds:
    (i)~there exists an integer $p \geq 0$ such that
    $\Pi_R \cup D \models F^{\phi}(\tau-p)$
    and 
    $\phi(\Pi_S) \cup D' \cup \{ F^{\phi}(\tau-p) \} \models G^{\phi}(\tau)$;
    (ii)~there exists an integer $q \geq 0$ such that
    $\Pi_R \cup D \models F^{\psi}(\tau-q)$ and
    $\psi(\Pi_T) \cup D' \cup \{ F^{\psi}(\tau-q) \} \models G^{\psi}(\tau)$.
    \smallskip\par

    By \sreogoneclaimref{claim:sre-og-0},
    it follows that 
    $\Pi_S \cup D' \cup \{ F(\tau-p) \} \models G(\tau)$ or 
    $\Pi_T \cup D' \cup \{ F(\tau-q) \} \models G(\tau)$ 
    by the construction of $\phi(\Pi_S)$ and $\psi(\Pi_T)$.
    By the inductive hypothesis, 
    either 
    there is a word
    $w = \langle a_1, \dots, a_n \rangle$ in $\lang(S)$  
    such that $F(\tau-n) \in D' \cup \{ F(\tau-p) \}$, and 
    $A_{a_i}(\tau-n+i-1) \in D'$ for each $1 \leq i \leq n$,
    or there is a word
    $w' = \langle b_1, \dots, b_{n'} \rangle$ in $\lang(T)$  
    such that $F(\tau-n') \in D' \cup \{ F(\tau-q) \}$, and 
    $A_{b_i}(\tau-n'+i-1) \in D'$ for each $1 \leq i \leq n'$.
    Note that in the former case $n=p$ and in the latter case $n'=q$, since
    $D'$ contains no $F$-fact.
    Finally, note that both $w$ and $w'$ are in $\lang(R)$, and it is easy to 
    see that in both cases $D$ satisfies the required properties.

    \smallskip\par\noindent
    \emph{Inductive~case~2.}
    It is the case where $R = S \circ T$, for $S$ and $T$ SREs.
    Let $\Pi_S$ and $\Pi_T$ be the programs for $S$ and $T$, respectively.
    Let $D'$ be the dataset consisting of each $A_{\sigma}$-fact in $D$ with 
    $\sigma \in \Sigma$.
    We have that $\Pi_R \cup D \models G^{\psi}(\tau)$, and hence the following
    claim holds by the construction of $\Pi_R$.

    \sreogoneclaim{claim:sre-og-1}
    There exists an integer $p \geq 0$ such that
    $\psi(\Pi_T) \cup D' \cup \{ F^{\psi}(\tau-p) \} \models G^{\psi}(\tau)$
    and $\Pi_R \cup D \models F^{\psi}(\tau-p)$.
    \smallskip\par

    Furthermore, 
    we have that $\Pi_R \cup D \models F^{\psi}(\tau-p)$ implies 
    $\Pi_R \cup D \models G^{\phi}(\tau-p)$---see 
    rule~\eqref{eq:lower-ic2-rule2}---and hence the following claim holds
    by the construction of $\Pi_R$.

    \sreogoneclaim{claim:sre-og-2}
    There exists an integer $q \geq 0$ such that
    $\phi(\Pi_S) \cup D' \cup \{ F^{\phi}(\tau-p-q) \} \models
    G^{\phi}(\tau-p)$
    and $\Pi_R \cup D \models F^{\phi}(\tau-p-q)$.
    \smallskip\par

    By \sreogoneclaimref{claim:sre-og-1}, we have that
    $\Pi_T \cup D' \cup \{ F(\tau-p) \} \models G(\tau)$,
    and hence by the inductive hypothesis
    there exists a word $w = \langle a_1, \dots, a_n \rangle$ in $\lang(T)$
    such that $F(\tau-n) \in D' \cup \{ F(\tau-p) \}$ and 
    $A_{a_i}(\tau-n+i-1) \in D'$ for each $1 \leq i \leq n$.
    It follows that $n=p$, since $D'$ contains no $F$-fact.
    By \sreogoneclaimref{claim:sre-og-2}, we have that
    $\Pi_S \cup D \models G(\tau-p)$,
    and hence by the inductive hypothesis
    there exists a word $w' = \langle b_1, \dots, b_{n'} \rangle$ in
    $\lang(S)$
    such that $F(\tau-p-n') \in D' \cup \{ F(\tau-p-q) \}$ and 
    $A_{b_i}(\tau-p-n'+i-1) \in D'$ for each $1 \leq i \leq n'$.
    It follows that $n'=q$, since $D'$ contains no $F$-facts.
    Furthermore, 
    $\Pi_R \cup D \models F^{\phi}(\tau-p-q)$ by \sreogoneclaimref{claim:sre-og-2}
    again, and hence $F(\tau-p-q) \in D$.
    Finally, note that $w'w \in \lang(R)$, and hence $w'w$ and $D$ are the
    required word and dataset, respectively.
    
    \smallskip\par\noindent
    \emph{Inductive~case~3.}
    It is the case where $R = S^+$.
    Let $\Pi_S$ be the program for $S$.
    It is easy to see from the construction of $\Pi_R$ that the following claim
    holds by the inductive hypothesis.

    \sreogoneclaim{claim:sre-og-3}
    Assume that $\Pi_R \cup D \models G^{\phi}(\tau')$ with
    $\tau' \leq \tau$. Then, there exists a word 
    $\langle \sigma_1, \dots, \sigma_n \rangle \in \lang(S)$ such that
    $\Pi_R \cup D \models F^{\phi}(\tau'-n)$ and
    $A_{\sigma_i}(\tau'-n+i-1) \in D$ for each $1 \leq i \leq n$.
    \par\smallskip

    Given the previous claim, we can prove the following one.

    \sreogoneclaim{claim:sre-og-4}
    There exist words $w_1, \dots, w_k$ for $k > 0$ such that, 
    for each $1 \leq i \leq k$, it holds that:
    \begin{itemize}
        \item
            $w_i = \langle \sigma_1^i, \dots, \sigma_{n_i}^i \rangle \in
            \lang(S)$,
        \item
            $\Pi_R \cup D \models F^{\phi}(\tau-\sum_{j=i}^k n_j)$,
        \item
            $A_{\sigma_j^i}(\tau+j-1-\sum_{\ell=i}^k n_{\ell}) \in D$ for each
            $1 \leq j \leq n_i$,
        \item 
            $F(\tau-\sum_{i=1}^k n_i) \in D$.
    \end{itemize}
     
    We omit the proof of \sreogoneclaimref{claim:sre-og-4},
    since it would be essentially similar to the one of 
    \sreogoneclaimref{claim:sre-og-7}, i.e.,
    it would consist in showing that the claim follows by repeated
    application of \sreogoneclaimref{claim:sre-og-3}---which is again similar to
    \sreogoneclaimref{claim:sre-og-6}.
    One major difference with \sreogoneclaimref{claim:sre-og-7} is the last
    point, for which we would need to argue that each derivation of $G(\tau)$
    from $\Pi_R \cup D$ has a leaf labelled with $F(\tau-\sum_{i=1}^k n_i)$
    because rule~\eqref{eq:lower-ic3-rule1} is the only rule in $\Pi_R$
    having no IDB atom in its body.
    \smallskip\par

    Then, the desired word is $w_1 \cdots w_k$, with $k$ each $w_i$ as in 
    \sreogoneclaimref{claim:sre-og-4}.

    \smallskip\par\noindent
    \emph{Inductive~case~4.~}
    It is the case where $R = S^k$ for $k \geq 2$ and $S$ an SRE.
    Let $\Pi_S$ be the program for $S$.
    Note that, as mentioned at the beginning of this proof, in this case we have
    to take $\Pi_\mathrm{succ}$ into account.
    Let $D'$ be the dataset consisting of each $A_{\sigma}$-fact in $D$ with 
    $\sigma \in \Sigma$.
    It is easy to see from the construction of $\Pi_R \cup \Pi_\mathrm{succ}$ 
    that the following claim holds.

    \sreogoneclaim{claim:sre-og-5}
    Assume that 
    $\Pi_R \cup \Pi_\mathrm{succ} \cup D \models G'(\vec{b}, \tau')$
    with $\tau' \leq \tau$.
    There exists an integer $p \geq 0$ such that
    $\Pi_R \cup D' \cup \{ F'(\vec{b}, \tau'-p) \} \models G'(\vec{b}, \tau')$ 
    and $\Pi_R \cup \Pi_\mathrm{succ} \cup D \models F'(\vec{b},\tau'-p)$.
    \par\smallskip

    We use the former claim to prove the next one.

    \sreogoneclaim{claim:sre-og-6}
    Assume that 
    $\Pi_R \cup \Pi_\mathrm{succ} \cup D \models G'(\vec{b},\tau')$ with 
    $\tau' \leq \tau$.
    Then, there exists a word 
    $\langle \sigma_1, \dots, \sigma_n \rangle \in \lang(S)$ such that
    $\Pi_R \cup \Pi_\mathrm{succ} \cup D \models F'(\vec{b},\tau'-n)$ and
    $A_{\sigma_i}(\tau'-n+i-1) \in D$ for each $1 \leq i \leq n$.

    We prove \sreogoneclaimref{claim:sre-og-6}.
    By \sreogoneclaimref{claim:sre-og-5}, there exists $n \geq 0$ such that
    $\Pi_R \cup \Pi_\mathrm{succ} \cup D \models F'(\vec{b},\tau'-n)$ and
    $\Pi_R \cup D' \cup \{ F'(\tau'-n) \} \models G'(\tau')$.
    It follows that 
    $\Pi_S \cup D' \cup \{ F(\tau'-n) \} \models G(\tau')$
    by the construction of $\Pi_R$.
    By the inductive hypothesis,
    there exists a word
    $\langle \sigma_1, \dots, \sigma_{n'} \rangle \in \lang(S)$ such that
    $F(\tau'-n') \in D' \cup \{ F(\tau'-n) \}$ and
    $A_{\sigma_i}(\tau'-n'+i-1) \in D'$ for each $1 \leq i \leq n'$.
    Since $D'$ contains no $F$-fact, we have that $n=n'$,
    and hence $D$ is as required.

    This concludes the proof of \sreogoneclaimref{claim:sre-og-6}.
    \par\smallskip

    Again, we use the former claim for proving the next one.

    \sreogoneclaim{claim:sre-og-7}
    There exist words $w_1, \dots, w_k$ such that,
    for each $1 \leq i \leq k$,
    it holds that:
    \begin{itemize}
        \item
            $w_i = \langle \sigma_1^i, \dots, \sigma_{n_i}^i \rangle \in
            \lang(S)$,
        \item
            $\Pi_R \cup \Pi_\mathrm{succ} \cup D \models
                F'(\vec{b},\tau-\sum_{j=i}^k n_j)$
            where $\vec{b}$ is the binary encoding of $i-1$,
        \item
            $A_{\sigma_{j}^i}(\tau+j-1-\sum_{\ell=i}^k n_{\ell}) \in D$ for each
            $1 \leq j \leq n_i$.
    \end{itemize}

    We prove \sreogoneclaimref{claim:sre-og-7} by induction on $i$ from $k$ 
    to $1$.

    In the base case $i = k$.
    As assumed above, we have that 
    $\Pi_R \cup \Pi_\mathrm{succ} \cup D \models G(\tau)$.
    It follows that 
    $\Pi_R \cup \Pi_\mathrm{succ} \cup D \models G'(\vec{b}, \tau)$
    where $\vec{b}$ is the binary encoding of $k-1$ by the construction of
    $\Pi_\mathrm{succ}$---see rule~\eqref{eq:lower-ic4-rule2}.
    By \sreogoneclaimref{claim:sre-og-6}, there exists a word
    $\langle \sigma_1, \dots, \sigma_n \rangle \in \lang(S)$ such that 
    $\Pi_R \cup \Pi_\mathrm{succ} \cup D \models F'(\vec{b},\tau-n)$ and
    $A_{\sigma_i}(\tau-n+i-1) \in D$ for each $1 \leq i \leq n$.
    In the inductive case $1 \leq i < k$, and we assume that the claim 
    holds if we replace $i$ with $i+1$.
    Let $N = \sum_{j=i+1}^k n_j$.
    By the inductive hypothesis, we have that 
    $\Pi_R \cup \Pi_\mathrm{succ} \cup D \models F'(\vec{b},\tau-N)$
    where $\vec{b}$ is the binary encoding of $i$.
    It follows that
    $\Pi_R \cup \Pi_\mathrm{succ} \cup D \models G'(\vec{c},\tau-N)$
    with $\vec{c}$ the binary encoding of $i-1$ by 
    the construction of $\Pi_\mathrm{succ}$---see 
    rule~\eqref{eq:lower-ic4-rule3}.
    By \sreogoneclaimref{claim:sre-og-6},
    there exists a word $\langle \sigma_1, \dots, \sigma_n \rangle$ in
    $\lang(S)$ such that 
    $\Pi_R \cup \Pi_\mathrm{succ} \cup D \models F'(\vec{c},\tau-N-n)$
    and $A_{\sigma_i}(\tau-N-n+i-1) \in D$ for each $1 \leq i \leq n$.

    This concludes the proof of \sreogoneclaimref{claim:sre-og-7}.
    \par\smallskip

    By \sreogoneclaimref{claim:sre-og-7},
    there exists a word 
    $w = \langle \sigma_1, \dots, \sigma_n \rangle \in \lang(R)$
    such that 
    $\Pi_R \cup \Pi_\mathrm{succ} \cup D \models F'(\vec{\bar{0}}, \tau - n)$ 
    and $A_{\sigma_i}(\tau - n) \in D$.
    Furthermore, 
    $\Pi_R \cup \Pi_\mathrm{succ} \cup D \models F'(\vec{\bar{0}}, \tau - n)$
    implies $F(\tau - n) \in D$---see rule~\eqref{eq:lower-ic4-rule1}.
    Therefore $w$ and $D$ are as required.
\end{proof}

\subsubsection{Proof of the Main Claim}

We finally show Theorem~\ref{thm:bodcontainment-expspacehard}, using the query 
construction given above.
Note that $\mathcal{O}$ is the class of datasets over objects from a given (but
arbitrary, and possibly empty) set of objects $O$; and also that $\lang(R)$
denotes the language of an SRE $R$.

\bodcontainmentexpspacehard*
\begin{proof}
    It suffices to show that $\containment_\mathcal{O}^{\fpclass}$ is
    $\expspacehard$, since it is $\logspace$-reducible
    to $\window_\mathcal{O}^{\fpclass}$ by 
    Theorem~\ref{th:window-containment-interreducible-prev}.
    We show a $\logspace$-computable many-one reduction $\varphi$ from the
    containment problem for succinct regular expressions (SREs) to
    $\containment_\mathcal{O}^{\fpclass}$.
    Then, the claim of the theorem follows from the fact that
    SRE containment is \expspacehard{}---see, e.g., \cite{sipser}.

    An instance $I$ of the containment problem for SREs is a pair of 
    SREs $R_1$ and $R_2$.
    Let $Q_{R_1}$ and $Q_{R_2}$ be the queries for $R_1$ and $R_2$ built as
    described above.
    Then, $\varphi$ maps $I$ to $\langle Q_{R_1}, Q_{R_2} \rangle$.
    Such queries can clearly be computed in logarithmic space,
    and hence the same holds for $\varphi$.
    We argue next that $\varphi$ is correct, i.e., 
    $\lang(R_1) \subseteq \lang(R_2)$ iff
    $Q_{R_1} \sqsubseteq_\mathcal{O} Q_{R_2}$.

    Assume $\lang(R_1) \subseteq \lang(R_2)$.
    Then we show $Q_{R_1} \sqsubseteq_\mathcal{O} Q_{R_2}$.
    Let $\tau$ be a time point and let $D$ be a dataset in $\mathcal{O}$
    such that $G(\tau) \in Q_{R_1}(D)$.
    Hence, we have to show that $G(\tau) \in Q_{R_2}(D)$.
    By Claim~\ref{claim:sreog-1}, there exists a word 
    $w = \langle \sigma_1, \dots, \sigma_n \rangle \in \lang(R_1)$ such that
    $F(\tau-n) \in D$, and each $A_{\sigma_i}(\tau-n+i-1) \in D$ for 
    $1 \leq i \leq n$.
    It follows that $w \in \lang(R_2)$ since $\lang(R_1) \subseteq \lang(R_2)$ 
    by our assumption,
    and hence $G(\tau) \in Q_{R_2}(D)$ by Claim~\ref{claim:sreog-2}.

    For the converse, assume $Q_{R_1} \sqsubseteq_\mathcal{O} Q_{R_2}$.
    Then we show $\lang(R_1) \subseteq \lang(R_2)$.
    Let $w = \langle a_1, \dots, a_n \rangle$ be a word in 
    $\lang(R_1)$.
    Let $D$ be the dataset consisting of $F(0)$ and $A_{a_i}(i-1)$ for each
    $1 \leq i \leq n$.
    Note that $D \in \mathcal{O}$, since $D$ mentions no objects.
    It follows that $G(n) \in Q_{R_1}(D)$ by Claim~\ref{claim:sreog-2},
    and hence $G(n) \in Q_{R_2}(D)$ since we have assumed that 
    $Q_{R_1} \sqsubseteq_\mathcal{O} Q_{R_2}$.
    By Claim~\ref{claim:sreog-1},
    there exists a word $w' = \langle b_1, \dots, b_{n'} \rangle$ in 
    $\lang(R_2)$ such that $F(n-n') \in D$ and $A_{b_i}(n-n'+i-1)$ for each
    $1 \leq i \leq n'$.
    We have that $n' = n$, since $F(0)$ is the only $F$-fact in $D$ by 
    construction.
    Furthermore, 
    we have that $b_i = a_i$ for each $1 \leq i \leq n$,
    since $A_{a_i}(i-1)$ is the only fact in $D$ of the form $A_{\sigma}(i-1)$ 
    for any $\sigma$.
    Therefore $w' = w$, and hence $w \in \lang(R_2)$ as required.
\end{proof}

\subsection{Proof of Theorem~\ref{thm:nr-lower-bounds}}

\begin{lemma} \label{lemma:ognr-containment-lower}
    $\containment^{\ognrclass}$ is \conptimehard{}.
\end{lemma}
\begin{proof}
    We prove the claim by giving a \logspace{}-computable many-one reduction
    $\varphi$ from \textsc{3-Sat} to the complement of
    $\containment^{\qclass}$ with $\qclass$ the propositional Datalog 
    subclass of $\ognrclass$.

    Now we describe the reduction $\varphi$.
    Let $\alpha$ be a $3$-CNF formula.
    Let $g$ be a fresh IDB nullary predicate---i.e., a propositional variable.
    Let $c_1, \dots, c_n$ be fresh IDB nullary predicates corresponding to the 
    clauses of $\alpha$.
    For each $1 \leq i \leq n$,
    let $l_{i,1}$, $l_{i,2}$ and $l_{i,3}$ be fresh EDB nullary predicates 
    corresponding to the literals of the clause corresponding to $c_i$.
    Let $Q_1$ be the query $\langle g, \Pi_1 \rangle$ where $\Pi_1$
    is the program consisting rule~\eqref{eq:propositional-lower-bound-rule-1},
    and each rule of the form~\eqref{eq:propositional-lower-bound-rule-2} for 
    $1 \leq i \leq n$ and $1 \leq j \leq 3$.
    Let $Q_2$ be the query $\langle g, \Pi_2 \rangle$ where $\Pi_2$
    is the program consisting of each 
    rule of the form~\eqref{eq:propositional-lower-bound-rule-3} for $l_{i,j}$ 
    and $l_{p,q}$ corresponding to complementary literals over the same
    propositional variable---e.g., literals $a$ and $\neg a$, where $a$ is a
    propositional variable.
    \begin{align}
        \label{eq:propositional-lower-bound-rule-1}
        c_1 \land \dots \land c_n &\to g  \\
        \label{eq:propositional-lower-bound-rule-2}
        l_{i,j} &\to c_i \\
        \label{eq:propositional-lower-bound-rule-3}
        l_{i,j} \land l_{p,q} &\to g
    \end{align}
    Note that $Q_1$ and $Q_2$ are clearly non-recursive propositional Datalog
    queries.
    Then, $\varphi$ maps $\alpha$ to $\langle Q_1, Q_2 \rangle$.
    We argue next that the reduction is correct, i.e., $\alpha$ is
    satisfiable iff $Q_1 \sqsubseteq Q_2$.
    In the following, in a slight abuse of notation, we identify any $c_i$
    with its corresponding clause, and any $l_{i,j}$ with its corresponding
    literal.

    Assume that $\alpha$ is satisfiable, i.e., that there exists a satisfying
    assignment $f$ for $\alpha$.
    We show that $Q_1 \not\sqsubseteq Q_2$.
    Let $D$ be the dataset consisting of each positive literal $l_{i,j}$ 
    whose propositional variable is made true by $f$, and each negative 
    literal $l_{i,j}$ whose propositional variable is made false by $f$.
    Clearly, for each pair of complementary literals $l_{i,j}$ and $l_{p,q}$
    sharing the same propositional variable, we have that
    $l_{i,j} \notin D$ or $l_{p,q} \notin D$.
    By the construction of $Q_2$, it follows that $\Pi_2 \cup D \not\models g$.
    Now, for each clause $c_i$ of $\alpha$, there exists a literal $l_{i,j}$ 
    made true by $f$, hence $l_{i,j} \in D$, and hence 
    $\Pi_1 \cup D \models c_i$
    by one of the rules of the form~\eqref{eq:propositional-lower-bound-rule-2}.
    It follows that $\Pi_1 \cup D \models g$ by 
    rule~\eqref{eq:propositional-lower-bound-rule-1}.
    Therefore $Q_1 \not\sqsubseteq Q_2$.

    For the converse, assume $Q_1 \not\sqsubseteq Q_2$.
    We show that $\alpha$ is satisfiable.
    There is a dataset $D$ such that $\Pi_1 \cup D \models g$ and
    $\Pi_2 \cup D \not\models g$.
    Let $f$ be the assignment for $\alpha$ such that $l_{i,j}$ is made true by
    $f$ iff $l_{i,j} \in D$.
    Since $\Pi_2 \cup D \not\models g$
    and by the rules of the form~\eqref{eq:propositional-lower-bound-rule-3},
    there is no pair of complementary literals $l_{i,j}$ and $l_{p,q}$ sharing 
    the same propositional variable and being both in $D$,
    and hence $f$ is a well-formed assignment for $\alpha$.
    Let $i \in [1,n]$.
    Since $\Pi_1 \cup D \models g$ and by 
    rule~\eqref{eq:propositional-lower-bound-rule-1},
    we have that $\Pi_1 \cup D \models c_i$,
    hence there exists $j$ such that $\Pi_1 \cup D \models l_{i,j}$,
    hence $l_{i,j} \in D$,
    and hence $l_{i,j}$ is made true by $f$, according to the construction 
    of $f$.
    Therefore $f$ is a satisfying assignment for $\alpha$.
\end{proof}

\begin{lemma} \label{lemma:boundedobject-nonrecursive-lower}
    $\containment_\mathcal{O}^{\nrclass}$ is \conexptimehard{} if $O$ contains
    at least two objects.
\end{lemma}
\begin{proof}
    Assume that $O$ contains at least two objects.
    We show that $\containment_\mathcal{O}^{\qclass}$ is already 
    \conexptimehard{} when $\qclass$ is the class $\dnrclass$ of non-recursive 
    \emph{Datalog} queries.
    In particular, we give a $\logspace$-computable many-one reduction 
    $\varphi$ from the exponential tiling problem to the complement of
    $\containment_\mathcal{O}^{\dnrclass}$.
    The claim then follows from the fact that the exponential tiling problem is 
    \nexptimecomplete{}---see, e.g., Section~3.2 of \cite{johnson1990catalog}.

    Our reduction  $\varphi$ is a straightforward adaptation of
    a reduction from the exponential tiling problem to the complement of 
    $\containment^{\dnrclass}$ given in~\cite{benedikt2010impact}. 
    Such a reduction does not directly apply to 
    $\containment_\mathcal{O}^{\dnrclass}$ because it makes use of an unbounded
    number of objects.
    Specifically, it requires datasets containing a number of objects which is 
    exponential in the size of the input to the exponential tiling problem,
    since each object represents one coordinate in the corridor that we are
    given to tile.
    In our reduction, instead, coordinates are encoded using just two objects. 
    The `downside' of such an encoding is that it makes use of predicates whose    
    arity is linear in the size of the input to the exponential tiling
    problem---but that is clearly irrelevant to our purposes.

    \smallskip\par\noindent
    \emph{The exponential tiling problem.}
    An instance $I$ of \textsc{ExpTiling} is a 5-tuple 
    $\langle n, r, H, V, T_0 \rangle$ where $n$ and $r$ are non-negative
    integers (coded in unary), $H$ and $V$ are subsets of $[1,r] \times [1,r]$,
    and $T_0$ is a total function from $[0,n-1]$ to $[1,r]$.
    A \emph{tiling} for $I$ is 
    a total function $T$ from $[0, 2^n-1] \times [0, 2^n-1]$ to $[1,r]$ 
    such that $T(0,j) = T_0(j)$ for each $j \in [0, n-1]$,
    $\langle T(i,j), T(i+1,j) \rangle \in H$ for each $i \in [0, 2^n - 2]$ 
    and each $j \in [0, 2^n-1]$,
    and
    $\langle T(i,j), T(i,j+1) \rangle \in V$ for each $i \in [0, 2^n-1]$ 
    and each $j \in [0, 2^n-2]$.
    Finally, \textsc{ExpTiling} holds for $I$ iff there is a tiling for $I$.
    \smallskip\par

    Given an instance $I$ of \textsc{ExpTiling} as described above,
    we define $\varphi(I)$ as the pair $\langle Q_1, Q_2 \rangle$ with 
    $Q_1 = \langle \mathit{goal}, \Pi_1 \rangle$ 
    and $Q_2 = \langle \mathit{goal}, \Pi_2 \rangle$, where
    $\mathit{goal}$ is a fresh nullary predicate,
    and $\Pi_1$ and $\Pi_2$ are described next.
    In the following construction, note that all the predicates considered are 
    rigid, and all the variables are of object sort.

    \smallskip\par\noindent 
    \emph{Construction of the left-hand program.}
    Next we build $\Pi_1$.
    Consider the following predicates.
    Let $\mathit{zero}$ and $\mathit{one}$ be fresh EDB unary predicates.
    Intuitively, 
    $\mathit{zero}(a)$ means that $a$ represents zero, and 
    $\mathit{one}(b)$ means that $b$ represents one.
    Furthermore, let $\mathit{eq}$ be a fresh IDB binary predicate.
    Intuitively, $\mathit{eq}(a,b)$ means that $a$ and $b$ represent the 
    same bit.
    Then, $\Pi_1$ contains rule~\eqref{eq:hardnessexpcontainment-rule-1}
    and rule~\eqref{eq:hardnessexpcontainment-rule-2}.
    \begin{align}
        \label{eq:hardnessexpcontainment-rule-1}
        \mathit{zero}(x) \land \mathit{zero}(y) &\to \mathit{eq}(x, y) \\
        \label{eq:hardnessexpcontainment-rule-2}
        \mathit{one}(x) \land \mathit{one}(y) &\to \mathit{eq}(x, y)
    \end{align}
    Consider the following predicates.
    Let $\mathit{tiledBy}_i$ be a fresh predicate of arity $2 \cdot n$ for each 
    $i \in [1,r]$.
    Intuitively, $\mathit{tiledBy}_i(\vec{a},\vec{b})$ means that the cell 
    with coordinates $\langle p, q \rangle$, where $p$ and $q$ are the 
    integers represented by $\vec{a}$ and $\vec{b}$, is tiled with the 
    $i$-th tile.
    Furthermore, let $V_i$ be a fresh predicate of arity $i + n$ for each 
    $i \in [0,n]$.
    Intuitively, $V_i(\vec{a},\vec{b})$ means that all the cells with 
    coordinates $\langle p,q \rangle$ are tiled for every integer $p$  
    whose least significant $i$ bits are represented by $\vec{a}$,
    and $q$ is the integer represented by $\vec{b}$.
    Then, $\Pi_1$ contains 
    each rule of the form~\eqref{eq:hardnessexpcontainment-rule-5} for 
    $i \in [1,r]$,
    and each rule of the form~\eqref{eq:hardnessexpcontainment-rule-6} for 
    $i \in [1,n]$.
    \begin{align}
        \label{eq:hardnessexpcontainment-rule-5}
        \mathit{tiledBy}_i(x_1, \dots, x_n, \vec{y}) 
        & \to V_n(x_1, \dots, x_n, \vec{y}) \\
        \label{eq:hardnessexpcontainment-rule-6}
        \textstyle
        V_i(x_1, \dots, x_i, \vec{y}) \land V_i(z_1, \dots, z_i, \vec{y}) 
        \land \bigwedge_{j=1}^{i-1} \mathit{eq}(x_j, z_j)
        \land \mathit{zero}(x_i) \land \mathit{one}(z_i) 
        & \to V_{i-1}(x_1, \dots, x_{i-1}, \vec{y})
    \end{align}
    Consider the following predicates.
    Let $H_i$ be a fresh predicate of arity $i$, for each $i \in [0,n]$.
    Intuitively, $H_i(\vec{a})$ means that all the cells with coordinates 
    $\langle p,q \rangle$ are tiled for every $p \in [0, 2^n-1]$ and 
    every integer $q$ whose least significant $i$ bits are represented by 
    $\vec{a}$.
    Then, $\Pi_1$ contains 
    rule~\eqref{eq:hardnessexpcontainment-rule-8},
    and each rule of the form~\eqref{eq:hardnessexpcontainment-rule-9} for 
    $i \in [1,n]$.
    \begin{align}
        \label{eq:hardnessexpcontainment-rule-8}
        V_0(x_1, \dots, x_n) & \to H_n(x_1, \dots, x_n) \\
        \label{eq:hardnessexpcontainment-rule-9}
        \textstyle
        H_i(x_1, \dots, x_i) \land H_i(y_1, \dots, y_i) 
        \land \bigwedge_{j=1}^{i-1} \mathit{eq}(x_j, y_j)
        \land \mathit{zero}(x_i) \land \mathit{one}(y_i) 
        & \to H_{i-1}(x_1, \dots, x_{i-1})
    \end{align}
    Finally, $\Pi_1$ contains rule~\eqref{eq:hardnessexpcontainment-rule-11}.
    \begin{equation}
        \label{eq:hardnessexpcontainment-rule-11}
        H_0 \to \mathit{goal}
    \end{equation}

    \smallskip\par\noindent 
    \emph{Construction of the right-hand program.}
    Next we build $\Pi_2$.
    First, $\Pi_2$ contains 
    rule~\eqref{eq:hardnessexpcontainment-rule-1} and 
    rule~\eqref{eq:hardnessexpcontainment-rule-2}.
    Then, $\Pi_2$ contains 
    rule~\eqref{eq:hardnessexpcontainment-rule-12}.
    \begin{equation}
        \label{eq:hardnessexpcontainment-rule-12}
        \mathit{zero}(x) \land \mathit{one}(x) \to \mathit{goal}
    \end{equation}
    Consider the following predicates.
    Let $\mathit{succ}$ be a fresh predicate of arity $2 \cdot n$. 
    Intuitively, $\mathit{succ}(\vec{a},\vec{b})$ means that 
    $\vec{a}$ and $\vec{b}$ represent two integers $p$ and $q$, respectively,
    such that $p + 1 = q$.
    Let $u_1, \dots, u_n, y, z$ be fresh variables.
    Then, $\Pi_2$ contains 
    rule~\eqref{eq:nr-binary-counting-1},
    rule~\eqref{eq:nr-binary-counting-2},
    and
    each rule of the form~\eqref{eq:nr-binary-counting-3} for $0 \leq i < n$,
    where $\vec{u}$ is the $i$-tuple $\langle u_1, \dots, u_i \rangle$, 
    $\vec{v}$ is the $(n-i-1)$-tuple $\langle v, \dots, v \rangle$, 
    and $\vec{w}$ is the $(n-i-1)$-tuple $\langle w, \dots, w \rangle$.
    \begin{align}
        \label{eq:nr-binary-counting-1}
        \mathit{zero}(x) &\to B(x) \\
        \label{eq:nr-binary-counting-2}
        \mathit{one}(x) &\to B(x) \\
        \label{eq:nr-binary-counting-3}
        \mathit{zero}(w) \land \mathit{one}(v) \land
        \textstyle 
        \bigwedge_{j=1}^i B(u_j) &\to 
        \mathit{succ}(\vec{u}, w, \vec{v}, \vec{u}, v, \vec{w})
    \end{align}
    In the rest of the construction, consider the following fresh variables.
    Let $\vec{x} = \langle x_1, \dots, x_n \rangle$,
    let $\vec{x}' = \langle x_1', \dots, x_n' \rangle$,
    let $\vec{y} = \langle y_1, \dots, y_n \rangle$,
    and let $\vec{y}' = \langle y_1', \dots, y_n' \rangle$.
    Then, $\Pi_2$ contains 
    each rule of the form~\eqref{eq:hardnessexpcontainment-rule-17} 
    for $i,j \in [1,r]$ with $i \neq j$.
    \begin{equation}
        \label{eq:hardnessexpcontainment-rule-17}
        \textstyle
        \bigwedge_{k=1}^n \mathit{eq}(x_k, x_k') \land 
        \bigwedge_{k=1}^n \mathit{eq}(y_k, y_k') \land 
        \mathit{tiledBy}_i(\vec{x}, \vec{y}) \land 
        \mathit{tiledBy}_j(\vec{x}', \vec{y}') \to \mathit{goal}
    \end{equation}
    Then, $\Pi_2$ contains 
    each rule of the form~\eqref{eq:hardnessexpcontainment-rule-18} 
    for $j,k \in [1,r]$ with $\langle j, k \rangle \notin V$,
    and each rule of the form~\eqref{eq:hardnessexpcontainment-rule-19} 
    for $j,k \in [1,r]$ with $\langle j, k \rangle \notin H$.
    \begin{align}
        \label{eq:hardnessexpcontainment-rule-18}
        \textstyle
        \bigwedge_{i=1}^n \mathit{eq}(x_i, x_i') \land 
        \mathit{succ}(\vec{y}, \vec{y}') \land
        \mathit{tiledBy}_j(\vec{x}, \vec{y}) \land 
        \mathit{tiledBy}_k(\vec{x}', \vec{y}') & 
        \to \mathit{goal} \\
        \label{eq:hardnessexpcontainment-rule-19}
        \textstyle
        \bigwedge_{i=1}^n \mathit{eq}(y_i, y_i') \land 
        \mathit{succ}(\vec{x}, \vec{x}') \land
        \mathit{tiledBy}_j(\vec{x}, \vec{y}) \land 
        \mathit{tiledBy}_k(\vec{x}', \vec{y}') & 
        \to \mathit{goal}
    \end{align}
    Then, $\Pi_2$ contains each rule of the 
    form~\eqref{eq:hardnessexpcontainment-rule-20} 
    for $j \in [0, n-1]$ and each $k \in [1,r]$ with $k \neq T_0(j)$,
    where $A_i$ is $\mathit{zero}$ if the least significant $i$-th bit of the
    binary encoding of $j$ is $0$ and $\mathit{one}$ otherwise.
    \begin{equation}
        \label{eq:hardnessexpcontainment-rule-20}
        \textstyle
        \bigwedge_{i=1}^n A_i(x_i) \land \bigwedge_{i=1}^n \mathit{zero}(y_i) \land 
        \mathit{tiledBy}_k(\vec{x}, \vec{y}) \to \mathit{goal}
    \end{equation}

    \smallskip\par\noindent 
    \emph{Correctness of the reduction.}
    We now argue that the reduction $\varphi$ is correct, i.e., there exists a 
    tiling for $I$ iff $Q_1 \not\sqsubseteq_O Q_2$.
    We show the two implications separately.

    \smallskip\par\noindent
    $(\Rightarrow)$
    Assume that $T$ is a tiling for $I$.
    Let $\bar{0}$ and $\bar{1}$ be two distinct objects in $O$---intuitively 
    standing for $0$ and $1$.
    Let $D$ be the dataset consisting of
    $\mathit{zero}(\bar{0})$, 
    $\mathit{one}(\bar{1})$, 
    and each fact $\mathit{tiledBy}_i(\vec{a}, \vec{b})$
    for $T(p,q) = i$, 
    $\vec{a}$ the $n$-tuple of $\bar{0}$'s and $\bar{1}$'s encoding $p$, and
    $\vec{b}$ the $n$-tuple of $\bar{0}$'s and $\bar{1}$'s encoding $q$.
    It is easy to verify that $\mathit{goal} \in Q_1(D)$ and 
    $\mathit{goal} \notin Q_2(D)$.

    \smallskip\par\noindent
    $(\Leftarrow)$
    Assume that there is a dataset $D$ such that 
    $\mathit{goal} \in Q_1(D)$ and
    $\mathit{goal} \notin Q_2(D)$.

    For each tuple $\langle a_1, \dots, a_n \rangle$ of objects such that 
    $\mathit{zero}(a_i) \in D$ or 
    $\mathit{one}(a_i) \in D$ for each $i \in [1,n]$,
    let $\operatorname{bin}(\vec{a})$ be the number in $[0,2^n-1]$ whose
    $i$-th bit is $0$ if $\mathit{zero}(a_i) \in D$ and $1$ otherwise.
    Note that $\operatorname{bin}$ is a well-defined function,
    since for each object $a$, it holds that 
    $\mathit{zero}(a) \notin D$ or 
    $\mathit{one}(a) \notin D$,
    by rule~\eqref{eq:hardnessexpcontainment-rule-12} since 
    $\mathit{goal} \notin Q_2(D)$.

    Let $T$ be the relation consisting of each tuple $\langle p, q, i \rangle$
    for 
    $\mathit{tiledBy}_i(\vec{a},\vec{b}) \in D$, 
    $\operatorname{bin}(\vec{a}) = p$,
    and $\operatorname{bin}(\vec{b}) = q$.
    In order to show that $T$ is a tiling for $I$, we have to show that
    (i)~$T$ is a total function over the domain $[0,2^n-1] \times [0,2^n-1]$,
    (ii)~$\langle T(i,j), T(i+1,j) \rangle \in H$ for each $i \in [0,2^n-2]$ and 
    each $j \in [0,2^n-1]$,
    (iii)~$\langle T(i,j), T(i,j+1) \rangle \in V$ for each $i \in [0,2^n-1]$ and 
    each $j \in [0,2^n-2]$,
    and (iv)~$T(0,j) = T_0(j)$ for each $j \in [0,n-1]$.
    %
    %
    We have that (i) holds since $T$ is total by 
    rules~\eqref{eq:hardnessexpcontainment-rule-5}--\eqref{eq:hardnessexpcontainment-rule-11}
    and $T$ is functional by rule~\eqref{eq:hardnessexpcontainment-rule-17}.
    We have that (ii) holds by rule~\eqref{eq:hardnessexpcontainment-rule-18}.
    We have that (iii) holds by rule~\eqref{eq:hardnessexpcontainment-rule-19}.
    We have that (iv) holds by rule~\eqref{eq:hardnessexpcontainment-rule-20}.
\end{proof}

\thmnrlowerbounds*
\begin{proof}
    $\window^{\ognrclass}$ is $\conptimehard$ because 
    it is at least as hard as $\containment^{\ognrclass}$ by 
    Theorem~\ref{th:window-containment-interreducible-prev}, and the former
    problem is $\conptimehard$ by Lemma~\ref{lemma:ognr-containment-lower}.
    $\window^{\nrclass}_\mathcal{O}$ is $\conexptimehard$ if $O$ contains at
    least two objects because it is at least as hard as 
    $\containment^{\nrclass}_\mathcal{O}$ by 
    Theorem~\ref{th:window-containment-interreducible-prev}, and the former
    problem is $\conexptimehard$ by 
    Lemma~\ref{lemma:boundedobject-nonrecursive-lower}.
\end{proof}



\fi

\end{document}
